\documentclass[sn-mathphys,Numbered]{sn-jnl}% Math and Physical Sciences Reference Style
\usepackage{amssymb}
\usepackage{utfsym}
\usepackage{graphicx}%
\usepackage{multirow}%
\usepackage{amsmath,amssymb,amsfonts}%
\usepackage{amsthm}%
\usepackage{mathrsfs}%
\usepackage[title]{appendix}%
\usepackage{xcolor}%
\usepackage{textcomp}%
\usepackage{manyfoot}%
\usepackage{booktabs}%
\usepackage{algorithm}%
\usepackage{algorithmic}%
\usepackage{listings}%
\usepackage[utf8]{inputenc}
\usepackage{dutchcal}
\usepackage{tabularx} 
\usepackage{array}
\usepackage{makecell}
\usepackage{enumitem}
\usepackage{units}
\setenumerate[1]{itemsep=1pt,partopsep=0pt,parsep=\parskip,topsep=5pt}
\setitemize[1]{itemsep=1pt,partopsep=0pt,parsep=\parskip,topsep=5pt}
\setdescription{itemsep=1pt,partopsep=0pt,parsep=\parskip,topsep=5pt}
\theoremstyle{thmstyleone}%
\newtheorem{theorem}{Theorem}%  meant for continuous numbers
\theoremstyle{thmstyletwo}%
\theoremstyle{thmstylethree}%
\newtheorem{definition}{Definition}%

\begin{document}

\title[Article Title]{HetFS: A Method for Fast Similarity Search with Ad-hoc Meta-paths on Heterogeneous Information Networks}


\author[1,3]{\fnm{Xuqi} \sur{Mao}}\email{xqmao17@fudan.edu.cn}

\author[2,3]{\fnm{Zhenyi} \sur{Chen}}\email{zhenyichen20@fudan.edu.cn}
% % \equalcont{These authors contributed equally to this work.}

\author[1, 3]{\fnm{Zhenying} \sur{He}}\email{zhenying@fudan.edu.cn}

\author[1, 3]{\fnm{Yinan} \sur{Jing}}\email{jingyn@fudan.edu.cn}

\author[1, 3]{\fnm{Kai} \sur{Zhang}}\email{zhangk@fudan.edu.cn}

\author*[1,2,3]{\fnm{X. Sean} \sur{Wang}}\email{xywangCS@fudan.edu.cn}
% % \equalcont{These authors contributed equally to this work.}

\affil*[1]{\orgdiv{School of Computer Science}, \orgname{Fudan University}, \orgaddress{\state{Shanghai}, \country{China}}}

\affil[2]{\orgdiv{Software School}, \orgname{Fudan University},  \orgaddress{\state{Shanghai}, \country{China}}}

\affil[3]{\orgname{Shanghai Key Laboratory of Data Science},  \orgaddress{\state{Shanghai}, \country{China}}}


\abstract{Numerous real-world information networks form \textbf{H}eterogeneous \textbf{I}nformation \textbf{N}etworks (HINs) with diverse objects and relations represented as nodes and edges in heterogeneous graphs. Similarity between nodes quantifies how closely two nodes resemble each other, mainly depending on the similarity of the nodes they are connected to, recursively. Users may be interested in only specific types of connections in the similarity definition, represented as meta-paths, i.e., a sequence of node and edge types. Existing \textbf{H}eterogeneous \textbf{G}raph \textbf{N}eural \textbf{N}etwork (HGNN)-based similarity search methods may accommodate meta-paths, but require retraining for different meta-paths. Conversely, existing path-based similarity search methods may switch flexibly between meta-paths but often suffer from lower accuracy, as they rely solely on path information. This paper proposes HetFS, a \textbf{F}ast \textbf{S}imilarity method for ad-hoc queries with user-given meta-paths on \textbf{Het}erogeneous information networks. HetFS provides similarity results based on path information that satisfies the meta-path restriction, as well as node content. Extensive experiments demonstrate the effectiveness and efficiency of HetFS in addressing ad-hoc queries, outperforming state-of-the-art HGNNs and path-based approaches, and showing strong performance in downstream applications, including link prediction, node classification, and clustering.
}




\maketitle

\section{Introduction}\label{sec1}

Similarity search of nodes on information networks serves as the foundation for numerous data analytics techniques  \cite{zhang2020continuously, echihabi2022hercules, mccauley2018set, wang2021deep, willkomm2019efficient} and has wide-ranging applications, including online advertising \cite{dey2020p}, recommendation systems \cite{zhang2020graph}, biomedical analysis \cite{selvitopi2023extreme, pavlopoulos2018bipartite, yi2022graph}, spatial-temporal systems \cite{patroumpas2020similarity, athanasiou2019big}. Most real-world information networks are heterogeneous information networks (HINs) \cite{zheng2022semantic, liao2022contrastive}, characterized by the coexistence of edges connecting nodes of various types (structural information) and properties associated with each node (often in terms of unstructured information). These connections illustrate the intricate semantics inherent in networks. We use a movie information network, illustrated in Fig.~\ref{fig:illustrationExampleOfHIN}(a), as a running example. The network consists of actors, movies, and directors, each interconnected by different types of relations.
% As a result, the evolution of information networks with complex content, semantics, and structure makes measuring object similarity over HINs a necessary task.

\begin{figure}[h]
\centering\includegraphics[width=0.7\textwidth]{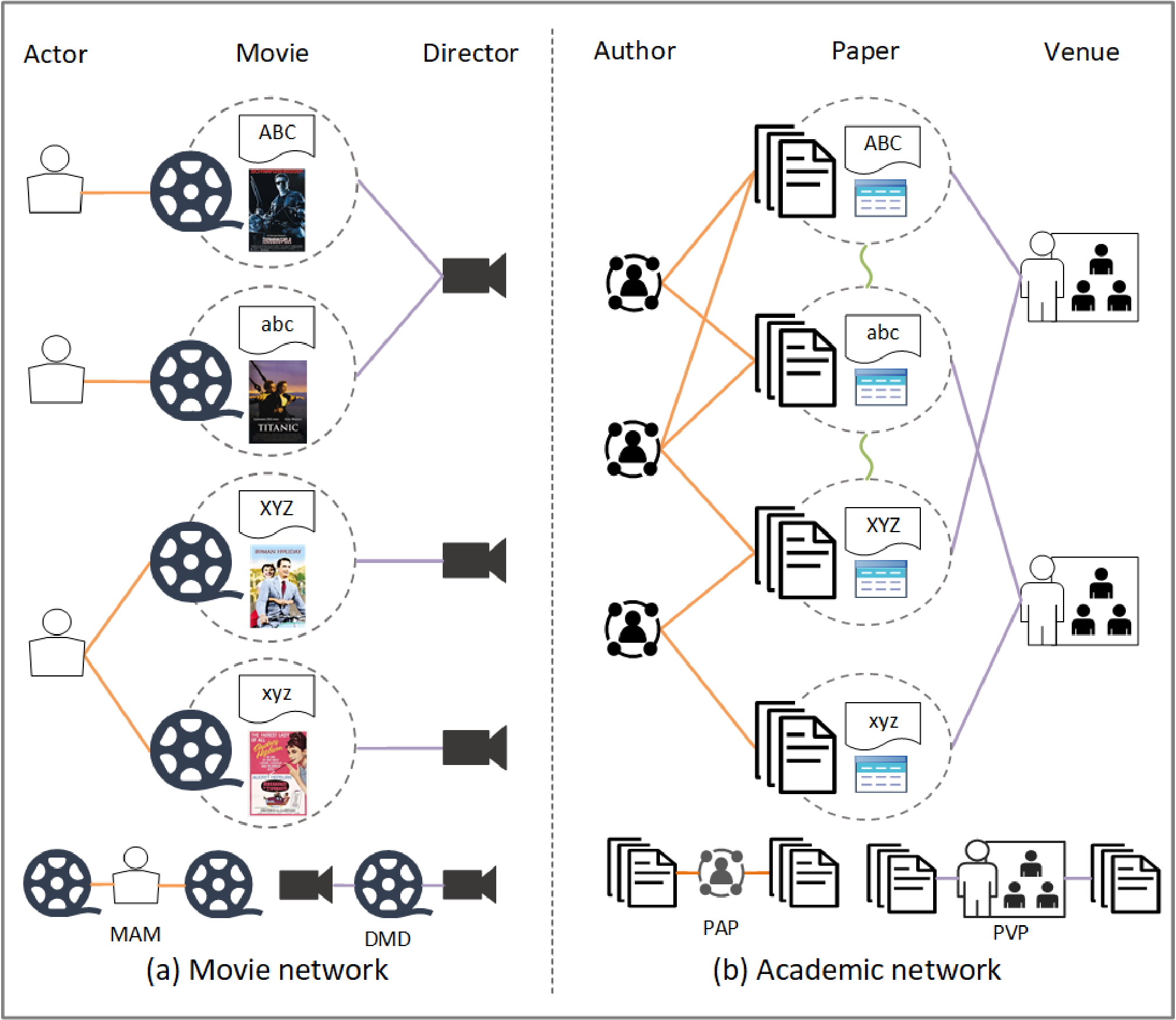}
\caption{Examples of HINs. (a) depicts a movie HIN comprising objects of different types such as actors (A), movies (M), and directors (D), each carrying varied properties including actor ID, actor name, author order, movie ID, movie name, director ID, and director name. These objects are interconnected through different types of relations, such as acting, acted-by, directed-by, and directing. ``MAM" (Movie-Actor-Movie) and ``DMD" (Director-Movie-Director) are two meta-path examples of it. (b) showcases objects and relations within an academic network, with ``PAP" (Paper-Author-Paper) and ``PVP" (Paper-Venue-Paper) as two example meta-paths.} \label{fig:illustrationExampleOfHIN}
\end{figure} 

% The similarity between two nodes is usually defined recursively by the similarity of the nodes themselves and that of the nodes they are respectively connected to. When the types of connection are restricted to specific ones at each step of the recursion, the similarity is said to be under a specific meta-path. Calculation of node similarity in HINs is challenging due to the varied structural and attribute-based relationships these entities share. In essence, meta-paths intuitively represent the connection pattern between the two nodes. For example, meta-paths like ``PAP” and ``PVP” in Fig. 1(a) indicate two papers with shared authorship or publication venues. Users often focus on the similarity of nodes associated with a particular meta-path, necessitating the handling of queries with ad-hoc meta-paths. User queries with ad-hoc meta-paths refer to the queries whose meta-paths are not predefined but tailored to specific querying needs, allowing systems to capture user search intentions and meet user expectations more accurately. To precisely align with user intent not only enhances user satisfaction but also improves the practical utility of the network. For instance, Table 1 displays movies similar to ``Terminator 2: Judgment Day" (abbreviated as ``Terminator 2" hereinafter)  under different meta-paths, 可以看出不同meta-path-guided下与同一步电影最相似的5部电影区别非常大。第一列电影（Jingle all the way等）都与Terminator 2 有共同的主演施瓦辛格，第二列电影（Titanic）都与Terminator 2都由卡梅隆执导。不同的meta-paths会导致搜索出的结果天差地别，highlighting the necessity of discerning user intent. 

The similarity between two nodes is generally defined recursively, considering both the nodes themselves and those to which they are connected. When the types of connection are restricted to specific ones at each step of the recursion, the similarity is said to be under a specific meta-path. Meta-paths effectively represent the patterns of connection between two nodes, encapsulating the structural and attribute-based relationships within the network. For example, meta-paths such as ``MAM" (Movie-Actor-Movie) and ``MDM" (Movie-Director-Movie) in Fig.~\ref{fig:illustrationExampleOfHIN}(a) illustrate movies linked either by shared actors or directors. Given their specificity, users often focus on node similarities that align with particular meta-paths, necessitating the ability to handle queries with ad-hoc meta-paths. These ad-hoc meta-paths, tailored to specific querying needs, are not predefined but are instead designed to capture user search intentions more precisely, thereby meeting user expectations more effectively. Accurately aligning query handling with user intent not only enhances user satisfaction but also increases the practical utility of the network.

For instance, when looking into the movie information network in our running example, we found that the movies most similar to ``Terminator 2: Judgment Day" (abbreviated as ``Terminator 2" hereinafter) exhibit considerable variation under different meta-paths. The first column in Table~\ref{tab:top5movies_terminator2} lists movies like ``Terminator," sharing similar lead actors with ``Terminator 2," whereas the second column includes films like ``True Lies," directed by similar directors with ``Terminator 2". A special case in similarity search is where all meta-paths are considered without restriction to provide a comprehensive view of similarity, named meta-path-free similarity. The last column of Table~\ref{tab:result_hetfs} gives the movies that are most similar to ``Terminator 2" under meta-path-free semantics in the movie information network of our running example. These examples underscore the importance of discerning user intent, as different meta-paths can lead to vastly different search results, highlighting the critical role of precise meta-path selection in achieving relevant outcomes. 

% Table generated by Excel2LaTeX from sheet 'merge'
\begin{table}[htbp]
\centering
\caption{\centering Top 5 similar movies of ``Terminator 2" under two meta-paths and a meta-path-free scenarios using HetFS.} \label{tab:top5movies_terminator2}
\begin{tabular}{clll}
\toprule
Rank  & meta-path: MAM   & meta-path: MDM   & meta-path-free \\
\midrule
1     & Terminator & True Lies & Terminator \\
2     & True Lies & Aliens & Terminator 3 \\
3     & Total Recall & The Abyss & True Lies \\
4     & Jingle All the Way & Titanic & Titanic \\
5     & End of Days & Ghosts of the Abyss & The Abyss \\
\bottomrule
\end{tabular}%
\label{tab:result_hetfs}%
\end{table}%

% While existing methods broaden the scope of similarity search, they often fall short when handling user queries with ad-hoc meta-paths in highly complex heterogeneous graphs. 

The above examples highlight the need for approaches that can effectively adapt to ad-hoc meta-paths, catering specifically to distinct user preferences. Despite advancements in visual preference learning \cite{wen2024unified}, contrastive learning \cite{zheng2023exploiting}, and attention networks \cite{pan2023iui, chen2023disentangling, wang2023intention} that have improved item recommendations and search accuracy by better aligning with user intentions \cite{rasch2023going, xu2023int}, there is a notable gap in research that focuses on leveraging ad-hoc meta-path configurations to enhance recommendation systems further.

Existing methods face limitations in dealing with user-specified meta-paths, either relying on pre-defined meta-paths or operating in a meta-path-free manner. GNN-based methods \cite{wang2019heterogeneous, hu2020heterogeneous, zhang2019heterogeneous} designed for HIN similarity search are often meta-path-free, leading to search results that may not align with user intentions. On the other hand, pre-defined meta-path methods \cite{fu2020magnn} fail to deliver accurate results if they don't match the user-given meta-paths. Training node embeddings online for each query is impractical due to time constraints.

Path-based similarity \cite{sun2011pathsim, wang2020howsim} search methods primarily use graph topology, resulting in lower accuracy compared to GNN-based methods. To enhance accuracy, we propose to calculate node similarity by integrating complex content information, intrinsic node characteristics, unique edge properties, and overall graph structure. For instance, rich textual content and diverse image data can capture more insights from HINs, enhancing similarity assessments. Different roles of an author (first author vs. last author) or an actor (lead vs. supporting roles) can influence similarity. Many existing path-based works overlook node centrality \cite{brin1998pagerank}, treating all nodes equally and assigning them equal weights. Centrality (a newcomer vs. a seasoned expert), which signifies node importance, significantly influences node similarity. Additionally, combining node properties with edge properties is crucial, as paths with similar nodes can differ if edge properties vary.

Considering the above limitations in the current approaches, this paper introduces HetFS, a Fast Similarity method with ad-hoc meta-paths on Heterogeneous information networks. HetFS integrates content information into path information by extracting and aggregating multi-level information, including content, node, edge, and structure. It utilizes type-specific transformation functions to project heterogeneous content, potentially of unequal dimensions, into a unified latent space. For node information, HetFS distinguishes between different types of nodes and assigns weights to node properties based on their types and degrees. To dig deeper into semantic information, HetFS extracts diverse semantics from edges to measure the varying impacts of heterogeneous edges. It's worth noting that, despite alignment between nodes and semantics, differences in node structure can result in varying influences on other nodes. Therefore, structure information still plays a crucial role in integrating graph topology. Finally, all the aforementioned information is aggregated to form the ultimate similarity calculation.

In summary, this work makes the following main contributions:
\begin{itemize}
\item It introduces HetFS to answer flexible queries with ad-hoc meta-paths, integrating the heterogeneous information of content, node, semantics, and structure.
\item HetFS projects heterogeneous content into a unified latent space to integrate content information with path information, thereby enhancing the final similarity accuracy.
\item HetFS incorporates a refined weighting approach, assigning weights to nodes based on both node type and centrality and to edges based on their contribution to adjacent nodes.
% \item HetFS scales the content similarity to enable aggregation with path similarity, yielding the final similarity.
\item We perform comprehensive experiments to showcase the accuracy and responsiveness of HetFS for queries with ad-hoc meta-paths in comparison to HGNNs and path-based approaches. Additionally, we show that HetFS has robust performance in various downstream applications, including link prediction, node classification, and clustering.

% HetFCF projects heterogeneous content into a unified latent space and aggregates them through a multi-layer transformation function to form the node representation.
% \item We optimize the model based on a random walk, which makes it capable of handling tasks on dynamic graphs.
\end{itemize}

% \begin{figure}[h]
% \centering\includegraphics[width=\textwidth]{excerptOfDBLPNew.png}
% \caption{An excerpt of DBLP. DBLP is an example of a HIN that consists of various types of objects, including authors (A), papers (P), and venues (V), as well as multiple types of relations, such as writing, citing, and publishing.} \label{fig:excerptOfDBLP}
% \end{figure}

% \begin{figure}[h]
% \centering\includegraphics[width=\textwidth]{workFlowNew.png}
% \caption{The Workflow of HetFS. Firstly, we clean the source dataset to obtain the contribution graph and index graph. The contribution graph can be automatically generated from the data graph or specified based on the input. Using the contribution graph as a basis, the query graph can be generated taking into account the input query from users. With the index graph and the query graph, our algorithm can calculate the similarity scores and obtain the result list based on these scores.} \label{fig:workflow}
% \end{figure}

The structure of this paper is as follows. In Section 2, we conduct a survey of the related work. Then we provide the necessary definitions in Section 3 and present our HetFS in Section 4. We report the results of our experimental studies in Section 5 and conclude the paper in Section 6.

\section{Related Works}\label{sec2}

Due to the ability of HINs to characterize complex information, much research has been dedicated to developing specialized graph mining techniques. Node similarity discovery plays a central role in these techniques.

The traditional algorithms for computing node similarity are path-based methods. PathSim \cite{sun2011pathsim} was the first method to calculate similarity between nodes in HINs, which introduces the innovative concept of meta-paths to assess relatedness between objects of the same type through symmetric paths. HeteSim \cite{shi2014hetesim} was a more general solution capable of quantifying the similarity between heterogeneous nodes regardless of whether the node types are the same or different. It adopted a path-constrained design to capture the semantic information present in the graph. Despite the awareness of different meta-paths in heterogeneous graphs, both PathSim and HeteSim treated each meta-path equally. These approaches overlooked the varied impact across different paths, leading to a loss of valuable semantic information. Howsim \cite{wang2020howsim} introduced a decay graph to encode the aggregation of similarities across various relations, which can capture semantics automatically from HINs. Despite Howsim recognizing the varying importance of different meta-paths, it solely utilized edge type information, neglecting other factors like node centrality. This limitation renders it incapable of handling complex situations, such as distinguishing between nodes of the same type connected by edges of the same type. Additionally, link-based methods do not leverage the content information from the HINs, resulting in less accurate results for similarity search.

With the advent of machine learning, the dominant approach to processing graphs has shifted towards using graph embedding techniques, which map graph elements to vectorized representations. DeepWalk \cite{perozzi2014deepwalk} introduced the method of word embedding to graph embedding by treating nodes as words and mapping the adjacent relationships of nodes into sentences, opening the door to graph representation learning. Node2vec \cite{grover2016node2vec} combined depth-first search strategy and breadth-first search strategy to sample nodes and generate biased second-order random walk sequences. Considering the heterogeneity of the graph, MetaPath2Vec \cite{dong2017metapath2vec} introduces a meta-path-based random walk, applying skip-gram to heterogeneous graphs. These methods either utilized only structural information or combined structural and semantic information without incorporating node content information. In terms of similarity accuracy, they all fell short of expectations.

Graph Neural Networks (GNNs) are widely recognized for their effective graph modeling and have consequently found application in HINs. They leverage deep neural networks to not only aggregate link information but also aggregate content information from neighboring nodes, thereby enhancing the power of the aggregated embeddings. Inspired by Graph Attention Network (GAT) \cite{velivckovic2017graph}, which initially introduced attention mechanisms for aggregating node-level information in homogeneous networks, HAN \cite{wang2019heterogeneous} proposed a two-level attention Heterogeneous Graph Neural Network (HGNN), incorporating both node-level and semantic-level attention. To leverage additional information from HINs, MAGNN \cite{fu2020magnn} employed three major components to encapsulate information from node content, intermediate semantic nodes, and multiple meta-paths. HetGNN \cite{zhang2019heterogeneous} jointly considers both structural and content information of each node without predefining meta-paths. However, when dealing with queries with ad-hoc meta-paths, these methods fall short. They either rely on meta-path-free approaches for automatic weighting across all data, making it challenging to respond to user preferences, or they calculate similarity based on predefined paths. The former struggles to adapt to user preferences, and the latter only produces good results when pre-defined meta-paths match the user-given ones. Aligning meta-paths is crucial for optimal results, but pre-defining all possible meta-paths is nearly impossible for the considerable scale and high complexity of HINs. Although theoretically feasible, retraining node embeddings is highly time-consuming and practically challenging to apply.

As a result, there is a lack of an existing similarity method to handle queries with ad-hoc meta-paths for HINs to integrate both content and path information holistically and efficiently.

\section{Preliminaries}
This section gives an overview of HIN and SimRank. Table~\ref{tab:notationTable} presents the notations used in this paper.

\begin{table}[h]
\caption{\centering Notation Table.}\label{tab:notationTable}
% \scriptsize
\begin{tabular}{cm{9cm}}
\toprule
% \makebox[2cm][c]{Notation} & {Description} \\
Notation & \makecell[c]{Description}\\
\hline
% $G$, $\mathcal{T}_G$ & the data graph, the network schema of $G$\\
$G$ & the data graph\\
$u$, $v$ & a node in the data graph\\
$e$ & an edge in the data graph\\
$V$, $E$ & the node set and the edge set of the data graph\\
$A$, $R$ & the node type and the edge type of the data graph\\
$\mathcal{A}$, $\mathcal{R}$ & the node type set and the edge type set of the data graph\\
$\phi$, $\psi$ & the type mapping function of the node and edge\\
$n$, $m$ &  the numbers of nodes and edges in $G$\\
% $N(u)$, $N_A(u)$, $N_R(u)$ & the neighbor set of $u$, the type-A neighbor set of $u$, and the neighbor set of $u$ along relation $R$ \\
$N(u)$, $N_R(u)$ & the neighbor set of $u$ and the neighbor set of $u$ along relation $R$ \\
$u^{\prime}$, $u_R^{\prime}$ & one of neighbors of $u$ and one of neighbors of node $u$ along relation $R$\\
$s(u, v)$ & the similarity score of $u$ and $v$\\
$c$ & the decay factor of SimRank\\
$c_n$ & the node decay factor\\
$t$ & a tour traversed during the random walk\\
% $s_p(u, v)$ & the similarity score of $u$ and $v$ under instance path $p$\\
$\mathcal{P}, \lvert\overline{\mathcal{P}}\rvert$ & the meta-path set and the average number of meta-path\\
$P, \lvert\overline{P}\rvert$ & the path set of a meta-path and the average number of paths in each meta-path\\
\bottomrule
\end{tabular}
\end{table}

\subsection{Heterogeneous Information Network}

\begin{definition}[Information Network]
An information network is a directed graph $G(V, E)$, where each node $v \in V$ is mapped to a specific node type $\phi(v) \in \mathcal{A}$ by $\phi : V \rightarrow \mathcal{A}$ and each edge $e \in E$ is mapped to a specific relation type $\psi(e) \in \mathcal{R}$ by $\psi : E \rightarrow \mathcal{R}$.
\end{definition}

An information network is said to be a Heterogeneous Information Network (HIN) when $|\mathcal{A}|$ \textgreater 1 or $|\mathcal{R}|$ \textgreater 1. If $|\mathcal{A}| = |\mathcal{R}| = 1$, then the information network is considered a homogeneous one.

% \noindent\textbf{Definition 2} Network Schema. The network schema of a HIN is a directed graph serving as a meta-template for HIN, defined on the node type set $\mathcal{A}$ and edge type set $\mathcal{R}$, represented as $\mathcal{T}_G(\mathcal{A}, \mathcal{R})$.

% % The network schema is akin to the ER (Entity-Relationship) model in database systems.

% The network schema provides valuable insights into the types of objects existing in the network graph as well as the potential relations between them. 

\begin{definition}[Meta-Path]
A meta-path $P$ is a sequence of node types in $\mathcal{A}$ and relation types in $\mathcal{R}$. It is typically represented with the notation of $A_1 \stackrel {R_1} {\longrightarrow} A_2 \stackrel {R_2} { \longrightarrow} \cdots \stackrel {R_l} {\longrightarrow} A_{l+1}$, where $A_1$ denotes the starting node type, $A_{l+1}$ denotes the ending node type, and each intermediate relation type $R_i$ denotes a connection type from node type $A_i$ and to node type $A_{i+1}$. To simplify notation, we can use node types along the meta-path to specify a meta-path if no more than one edge type exists connecting the identical pair of node types. 
\end{definition}

As an illustration, the co-author relation in DBLP, denoted by a length-2 meta path $A\stackrel{writing}{\longrightarrow}P\stackrel{written-by}{\longrightarrow}A$, can be represented in a shorter form $APA$ with no ambiguity.

% A reverse meta-path $P^{-1}$ serves as the reverse of $P$, whose relation is absolutely inverse to that of $P$. $P_1 = (A_1A_2...A_k)$ and $P_2 = ({A_1^{'}}{A_2^{'}}...{A_k^{'}})$, are considered connectable if and only if $A_k = A_1^{'}$. The connected meta-path $(A_1A_2...A_k{A_2^{'}}...{A_k^{'}})$ can be denoted as $P = (P_1P_2)$.

% Similarity score of a node pair $x \in A_1$ and $y \in A_l$ can be computed if there exists a meta-path $P = (A_1A_2...A_l)$ between them. The computation involves aggregating the similarity scores of each instance of the meta-path $P$. Besides, we denote path count as $PC$, which corresponds to the number of path instances of meta-path $P$ between $x$ and $y$, given by $PC = |p:p \in P|$.

\subsection{SimRank}

SimRank is a similarity measure defined on homogeneous information networks with a long history and broad applications across different domains. It is based on two fundamental principles: \textit{a}) two nodes are considered similar if connected to similar nodes; \textit{b}) a node has a similarity score of 1 with itself. Given an unweighted directed graph, ${G\,(\,V, E\,)}$, the SimRank score of two nodes $u \in V$ and $v \in V$ is defined recursively:
% \begin{small}
\begin{equation} 
	\label{(simrankEquation)}
	s(u, v)=
	\left\{
	\begin{array}{llr}
		1, & \text{if } u = v; & \\
		\frac{c}{|N(u)||N(v)|}\sum\limits_{u^{\prime} \in N(u)}\sum\limits_{v^{\prime} \in N(v)}s(u^{\prime}, v^{\prime}), & \text{otherwise.}
	\end{array}
	\right.
\end{equation}
% \end{small}
where $N(u)$ and $N(v)$ denote the neighbor set of node $u$ and $v$, respectively, with $c$ $\in$ (0,1) commonly set to 0.6 \cite{wang2021exactsim} or 0.8 \cite{jeh2002simrank}, is a decay factor to ensure that the similarity between different nodes does not equal 1. Due to the recursiveness of Eq.~\ref{(simrankEquation)}, SimRank can integrate similarities not only from their direct neighbors of $u$ and $v$ but also from their indirect neighbors (multi-hop neighbors), resulting in accurate, reliable, and meaningful outcomes.

To gain a more intuitive understanding of the calculated values, SimRank introduced a random surfer-pairs model \cite{jeh2002simrank}. This model interprets $s(u,v)$ as the expected first-meeting distance for two random surfers, each originating from nodes $u$ and $v$ and destinating at the same node $w_l$. A tour $t$ may be written as $t= \langle w_1, \cdots w_l(w_{\text{-}l}),\cdots, w_{\text{-}1}\rangle $, and $l$ is said to be the length of the tour. Intuitively, the first surfer walks from $w_1$ to $w_2, \cdots, w_l$ and the second surfer walks from $w_{\text{-}1}$ to $w_{\text{-}2}, \cdots, w_{\text{-}l}$, and we assume $w_l = w_{\text{-}l}$. When $l = 0$, we have $w_1 = w_{\text{-}1}$. For a tour $t$ traversed during the random surf from node $u$ and $v$, here $u = w_1$ and $v = w_{\text{-}i}$, $s(u,v)$ can be reformulated based on this random surfer-pairs model as:
\begin{equation} 
    \label{(random_walk)}
    s(u, v)=\sum\limits_{t}Pr(t)c^{l}
\end{equation}
where
\begin{equation} 
    \label{(Pr)}
    Pr(t)=
    \left\{
    \begin{array}{llr}
    1, & \text{if } l = 0; & \\
    \prod\limits_{w_i \in  t}\frac{1}{|N(w_i)|}, & \text{otherwise.}
    \end{array}
    \right.
\end{equation}
% \begin{equation} 
%     \label{(random_walk)}
%     s(u, v)=\sum\limits_{t:(a,b)\leadsto(w_l,w_l)}Pr(t)c^{l}
% \end{equation}
% \begin{equation} 
%     \label{(first_meeting_probability)}
% \end{equation}
and $t= \langle w_1, \cdots w_l(w_{\text{-}l}),\cdots, w_{\text{-}1} \rangle$ ranges over all the tours that the random surfers start from node $u (w_1)$ and $v (w_{\text{-}1})$ and reach a node $w_l (w_{\text{-}l})$, and $Pr(t)$ denotes the first meeting probability of tour $t$, where $N(w_i)$ denotes the set of neighbors of node $w_i$.

\section{Methodology}\label{sec3}

% In various domains, similar nodes are more likely to connect than dissimilar nodes. For example, customers loyal to similar brands tend to purchase similar products, and researchers publishing papers at similar conferences are likely to have connections. Consequently, the similarity between two nodes can be deduced from their shared references. Inspired by this intuition, 
% Inspired by the definition of SimRank and its random surfer-pairs model, we extend SimRank \cite{jeh2002simrank} to HINs by aggregating multi-level information along the random tours into the original definition. The abundant content within HINs offers intricate information. Neglecting this rich content information within heterogeneous graphs will inevitably result in computed similarity values that fail to capture the similarity between nodes fully. This is the reason why we combine content information and path information to propose our HetFS method.

Due to the effectiveness of SimRank \cite{jeh2002simrank} in capturing path information, we propose HetFS to extend SimRank to HINs by aggregating multi-level information and combining content information with path information to capture node similarity fully.

This section introduces HetFS, a fast similarity search method with ad-hoc meta-paths for HINs. HetFS integrates multi-level information from content, nodes, edges, and structure to enhance search capabilities, as depicted in Fig.~\ref{fig:architecture}, which illustrates its architecture. Note that while HetFS is capable of processing both symmetric and asymmetric meta-paths, it primarily focuses on symmetric meta-paths, aligning with the principles established by PathSim~\cite{sun2011pathsim} to ensure balanced and comparable visibility between nodes. Further details will be discussed in the following sections.

% HetFS computes similarity with a deep understanding of heterogeneous nodes. 
% The architecture of HetFS is illustrated in Fig. 2, depicting the process of aggregating content similarity to path similarity to generate the final similarity.

\begin{figure}[h]
\centering\includegraphics[width=\textwidth]{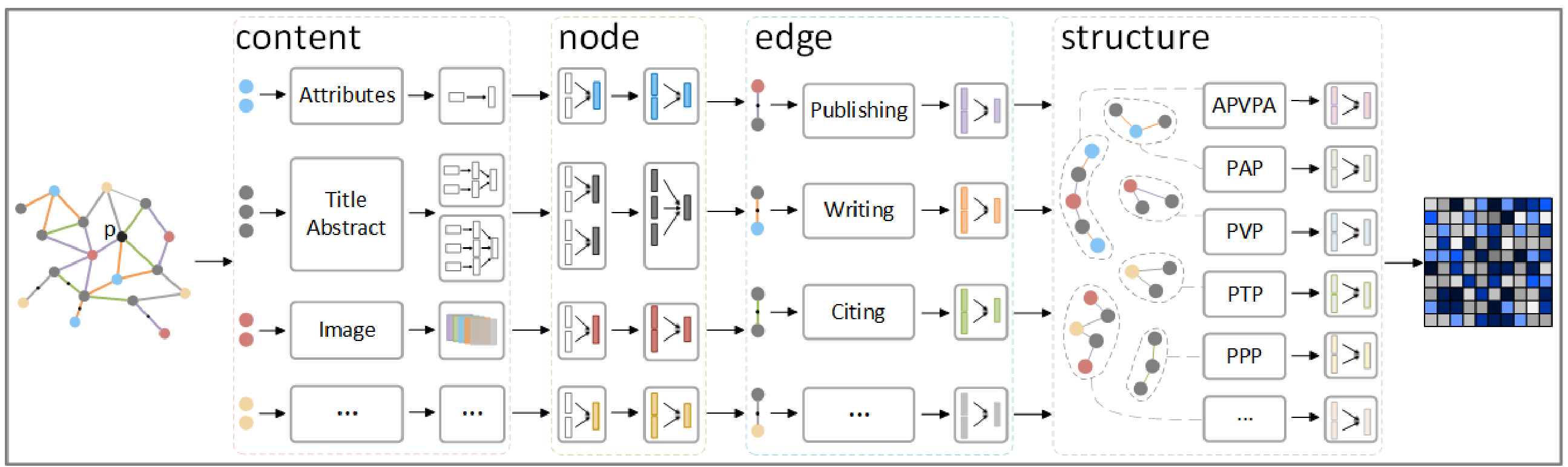}
\caption{The overall architecture of HetFS. HetFS first projects heterogeneous content information, e.g., textual data, from each node into a unified domain via a transformation function. Subsequently, the content information is integrated with node centrality, edge contribution, and structural topology to form the ultimate similarity computation.} \label{fig:architecture}
\end{figure}

% \subsection{Content Similarity}

\subsection{Content Information}

Nodes can contain diverse types of information as their content, such as textual data, image data, and more. Most existing methods overlook the content information carried by the nodes, which makes the information in the HIN not fully utilized, and the final result may not be satisfactory. HetFS extracts content information, integrating it into the path-based similarity, thereby enhancing the overall accuracy of the final similarity score.

Considering the unique content characteristic associated with each dataset, we project heterogeneous content $C_u$ of each node $u \in V$ into a given domain through a transformation function $f$. The $i$-th content in $C_u$ is denoted as $x_i(u)$. Note that $x_i(u)$ can undergo various processing depending on the type of content. For instance, we employ tf-idf to process textual content and CNN for image content. Specifically, for textual content, we begin by tokenizing each document, and then perform lemmatization and remove common stop words. Subsequently, we utilize Word2Vec \cite{mikolov2013efficient} to capture corpora characteristics, creating corpora-specific stop-word and synonym lists, which were then used to enhance the tokenization results of the content. Note that we generate individual stop-word lists and synonym lists for each textual dataset, with each dataset comprising a collection of text based on how it is associated with the nodes. For example, texts may be from a title or an abstract, forming two specific datasets. Finally, we calculate content score based on tf-idf. Formally, the content score of node $u$, represented as $\chi(u)$ can be formulated:
\begin{equation} 
    \label{(content score)}
    \chi(u)=\sum\limits_{i = 1, \cdots, |C_u|} f_i(x_i(u)) 
\end{equation}
where $C_u$ is the heterogeneous content of node $u$, and $f$ is the specific transformation function for each content in $C_u$.

% \subsection{Path Similarity}

\subsection{Node Information}
% \noindent\textbf{Node Information}

Node centrality $\alpha$ is a measure that gives the importance of a node. For instance, academic giants have higher academic influence compared to novice entrants. Inspired by PageRank, we say that a node is of high centrality if it is referenced by many nodes that themselves possess high centrality. Intuitively, a node that is well linked with other nodes is usually of high centrality.

To distinguish centrality with respect to nodes of different types, we calculate centrality separately for each node type and then aggregate the results. Based on the discussion above, we define the node centrality recursively as follows:
\begin{equation} 
    \label{nodecentrality}
    \alpha(u)=c_n \cdot \sum\limits_{R \in \mathcal{R}}\sum\limits_{u_R^{\prime}\in N_R(u)}\alpha(u_R^{\prime}) \cdot \mathbf{P}(u, u^{\prime})
\end{equation}
where $c_n$ is the node decay factor to deal with dangling links, $u_R^{\prime}$ is one of the neighbors of $u$ along relation $R$, and $\mathbf{P(u, u^{\prime})}$ denotes the transition factor as follows:
\begin{equation} 
    \label{(transition_factor)}
    \mathbf{P}(u, u^{\prime})=
    \left\{
    \begin{array}{llr}
    \frac{1}{|N_R(u)|}, & \text{if } u^{\prime} \in N_R(u); & \\
    0, & \text{otherwise.} 
    \end{array}
    \right.
\end{equation}
in which $N_R(u)$ denotes the neighbors of $u$ along relation $R$. Eq.~\ref{nodecentrality} is recursive, aligning with the intuition that a node's centrality is distributed evenly among its neighbors, thereby contributing to the centrality of the connected nodes. It can be computed by initializing with some random values and iterating the computation until convergence or reaching a specific iteration.

% $arctan(x)$

\subsection{Edge Information}

% In heterogeneous graphs, we observe variations not only in edge categories but also in the additional information carried by edges of the same category. For instance, a particular edge may carry a rank within its category of edges. The former is described in this paper as the global edge contribution $\mu$, while the latter as the local edge contribution $\lambda$. It is not difficult to understand that edges of different types convey different information. Regarding the latter, for example, the information conveyed to an article by the first author and the second author is also different. Customizing the weight of relations, referred to as the edge contribution, is crucial for HetFS, as weights embody the importance of different relations. 

% The inspiration for the global edge contribution comes from the well-known numerical statistic called TF-IDF, commonly used in information retrieval. 
% Given a network schema $T_G{(\mathcal{A, R})}$,
In HINs, the diverse types of edges convey distinct semantics, thereby contributing differently to the overall similarity score. By analyzing the statistics and structure of the data graph, we can derive the edge contribution for each edge type, which can also be adjusted based on user input and preferences. Given an HIN $G(V, E)$ with relation types $R$, a global edge contribution $\mu_R$ maps each relation type $R \in \mathcal{R}$ to a fraction:
\begin{equation} 
    \label{(contributionFunction)}
    \mu_R = RF(R) \cdot IRF(R)
\end{equation}
where $RF(R)$ and $IRF(R)$ are calculated based on the following heuristics:
\begin{enumerate}
    \item [1)] $RF(R)$ is the normalized relation frequency of relation $R$, which can be calculated as $\frac{m_R}{m}$, i.e., the fraction of the count of edges with relation $R$ in $G$ within the total count of edges in $G$, preventing the contribution coefficient from biasing toward frequently-appear edges.
    \item [2)] $IRF(R)$ is the inverse frequency of nodes associated with relation $R$, $IRF(R) = \ln \frac{n}{|n_R|}$, which is the Napierian logarithmic of the total count of nodes over the count of nodes involved with relation $R$. This count illustrates the differentiation capability among the data graph. In other words, the fewer neighbors of a relation $R$, the better a relation can differentiate a node. 
\end{enumerate}

By combining the edge contributions, we can construct a contribution graph $M(\mathcal{A}, \mathcal{R}, \mu_R)$, which is a directed graph with weighted edges, where $\mathcal{A}$ and $\mathcal{R}$ denotes the node type set and edge type set, respectively. This contribution graph illustrates how a node accumulates similarities from its neighbors across different relations, as well as how its neighbors contribute similarities to the node across different relations. Fig.~\ref{fig:contribution_graph} is an example of the contribution graph of the movie network. We can see that for the ``movie-to-movie'' similarity calculation, considering only edge contribution, 47\% similarity originates from shared actors, 41\% similarity from being part of the same film series, and 30\% from common directors.

\begin{figure}[h]
\centering\includegraphics[width=0.7\textwidth]{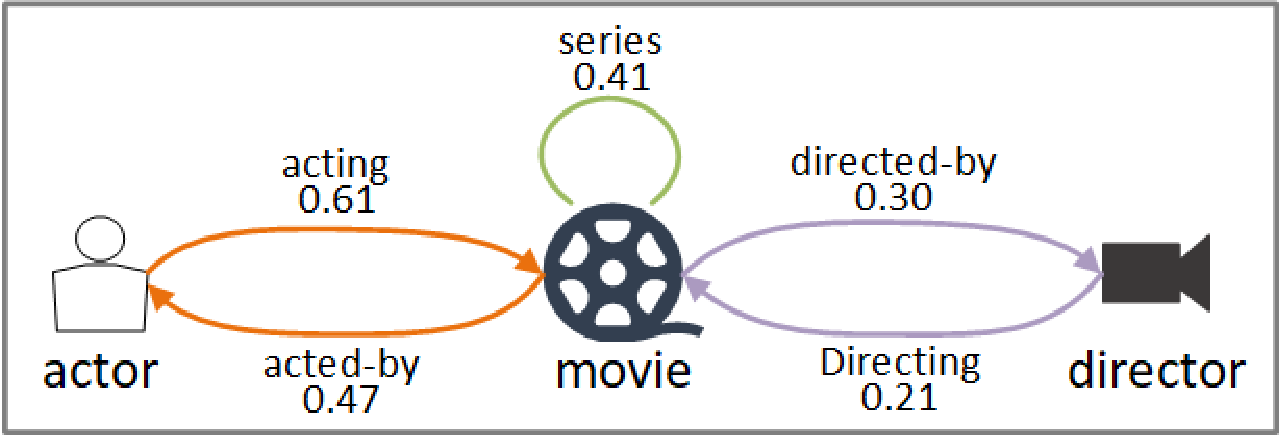}
\caption{The contribution graph of movie network.} \label{fig:contribution_graph}
\end{figure}

\subsection{Structure Information}
% \noindent\textbf{Structure information}

Similar to SimRank, we consider that information is distributed among the neighbors of a node. 
Here, we calculate the structure weight for each type of neighbor of a node separately. Given the heterogeneity of HINs, we define structure weight $\beta$ for each node $u$ based on the count of neighbors connected via type $R$ as follows:
\begin{equation} 
    \label{(strctureWeight)}
    \beta_R(u) = |N_R(u)|
\end{equation}
% \begin{equation} 
%     \label{(strctureWeight)}
%     \beta(u) = \sum\limits_{A\in \mathcal{A}}\frac{1}{|N_A(u)|}
% \end{equation}
where $|N_R(u)|$ is the count of the neighbors of $u$ along relation $R$. The structure weight ensures that the impact of a node on its neighbors is related to the of its neighbors. 

\subsection{Integration}
% \noindent\textbf{Integration}

Integrating content information, node information, edge information, and structure information gives the overall similarity. For a HIN $G(V, E)$, the similarity $s(u, v)$ of two nodes $u \in V$ and $v \in V$ is defined as:
% \begin{equation} 
%     \label{(pathSimilarity)}
%     s_p(u, v)=
%     \left\{
%     \begin{array}{lr}
%     0,  \hspace{20.8em} \rm{if  u = v} & \\
%     c\sum\limits_{R\in \mathcal{R}}\sum\limits_{u^{'} \in N(u)}\sum\limits_{v^{'} \in N(v)}\lvert\lvert \tanh{(\alpha \lambda)}\mu\rvert\rvert_{1} \beta s(u^{'}, v^{'}),  \hspace{0.8em} \rm{otherwise}
%     \end{array}
%     \right.
% \end{equation}
\begin{equation} 
    \label{power_method}
    s(u, v)=
    \left\{
    \begin{array}{llr}
    1, & \text{if } u = v; & \\
    \sum\limits_{R\in \mathcal{R}_{\mathcal{P}}}\sum\limits_{u^{\prime} \in N_R(u)}\sum\limits_{v^{\prime} \in N_R(v)} \frac{c \cdot \chi(u^{\prime}) \cdot \chi(v^{\prime}) \cdot \alpha(u^{\prime}) \cdot \alpha(v^{\prime}) \cdot \mu_R}{\beta_R(u^{\prime}) \cdot \beta_R(v^{\prime})} \cdot s(u^{\prime}, v^{\prime}), & \text{otherwise.}
    \end{array}
    \right.
\end{equation}
where $\mathcal{R}_{\mathcal{P}}$ denotes the edge types set corresponding to specific meta-paths, $c$ ($0 < c < 1$) is the decay factor inherited from SimRank, which ensures that no other node, except the node itself, has a similarity value of 1 with that node, $\chi(u^{\prime})$ and $\chi(v^{\prime})$ are the content information of $u^{\prime}$ and $v^{\prime}$, respectively, $\alpha(u^{\prime})$ and $\alpha(v^{\prime})$ are the node centrality of $u^{\prime}$ and $v^{\prime}$, respectively, $\mu(R)$ is the edge contribution of edge type $R$, $\beta(u^{\prime})$ and $\beta(v^{\prime})$ is the structure weight of $u^{\prime}$ and $v^{\prime}$, respectively. The similarity score is 1 if $u$ and $v$ are the same node by definition. Otherwise, Eq.~\ref{power_method} combines the similarities from $N(u)$ and $N(v)$ through each $R \in \mathcal{R}_{\mathcal{P}}$. 

% Next, we give the matrix representation for Eq.(6) using the linearization method. Given a HIN $G(V, E)$, the transition matrix $Q$ of transpose graph $G^T$ is defined:
% \begin{equation}
%     \label{(transitionMatrix)}
%     Q_{i, j} = 
%     \left\{
%     \begin{array}{llr}
%     \frac{\chi(u^{'}) \cdot \chi(v^{'}) \cdot \alpha(u^{'}) \cdot \alpha(v^{'}) \cdot \mu_R}{\beta_R(u^{'}) \cdot \beta_R(v^{'})}, & \text{if } \exists (i \rightarrow j) \in E; & \\
%     0, & \rm{otherwise.}
%     \end{array}
%     \right.
% \end{equation}
% where $I_R(j)$ and $I_R(i))$ represent the in-neighbor set of $j \in V$ and that of $i \in V$, respectively. $(i \rightarrow j) \in E$ represents $i$ and $j$ is connected over $G$ For a node type $A \in \mathcal{A}$, we define $S_A$ as the matrix equivalent to Eq.(6):
% \begin{equation}
%     \label{(HetFSMatrixRepresentation)}
%     S_A = max\{\sum_{R \in R(A)} c \cdot (Q \cdot S \cdot Q^T), I_n\} = (\sum_{R \in R(A)} c \cdot Q \cdot S \cdot Q^T) \vee I_n
% \end{equation}
% where $Q^T$ is the matrix transpose of $Q$, $I_n$ is the $n \times n$ identity matrix, $max\{,\}$ and $\vee$ are both element-wise maximum operator. In other words, the entry $s_A(i, j)$ of the matrix $C \vee D$ is given by $max\{C_{ij}, D_{ij}\}$.
\begin{algorithm}
\caption{Brute-force HetFS Computation}
\label{alg:bruteforce_computation}
\begin{algorithmic}[1]
\REQUIRE heterogeneous graph $G(V, E)$, number of iterations $l$, $\chi(u)$, $\alpha(u)$, $\beta_R(u)$, $\mu_R$, node $u$, $v$
\ENSURE similarity $s$
\FOR{$i \gets 0$ to $l-1$}
\IF{$u = v$}
\STATE $s(u,v) \gets 1$
\ELSIF{$i < l-1$}
\FOR{each $u^{\prime} \in N(u)$}
\FOR{each $v^{\prime} \in N(v)$}
\STATE $s(u,v) \gets s(u,v) + \frac{c \cdot \chi(u^{\prime}) \cdot \chi(v^{\prime}) \cdot \alpha(u^{\prime}) \cdot \alpha(v^{\prime}) \cdot \mu_R}{\beta_R(u^{\prime}) \cdot \beta_R(v^{\prime})} \cdot s(u^{\prime}, v^{\prime})$
\STATE $i \gets i+1$
\ENDFOR
\ENDFOR
\ELSE
\STATE $s(u,v) \gets 0$
\ENDIF
\ENDFOR
\RETURN $s(u,v)$
\end{algorithmic}
\end{algorithm}

Algorithm~\ref{alg:bruteforce_computation} provides the pseudocode for a brute-force HetFS computation. If $u = v$, their similarity is set to 1, as shown in Line 4. From Line 5 to Line 13,  brute-force HetFS iteratively calculates the similarity between the two nodes. When the iteration count exceeds the predefined threshold $l$, indicating that the nodes are no longer potentially similar, Line 13 sets their similarity to 0.

SimRank has substantial computational costs in both time and space. Similarly, the brute-force approach of HetFS also encounters significant challenges in terms of computational overhead. In the next section, we propose an efficient optimized strategy of HetFS based on the random surfer method.

% % For each node $u \in V$, if the accessed node during every iteration is equal to node $u$, the HetFS score is assigned as 1. Otherwise, the similarity score is computed using vector multiplication as described in Eq.10. 
% % The time complexity of this computation is $O(ln^2)$, and storing the similarity matrix requires $O(n^2)$ space.

\subsection{A Path Enumeration Strategy}

The random surfer-pairs model of SimRank provides an intuitive understanding of its computation process, and we extended it from homogeneous information networks to HINs to make HetFS more intuitive. 
 % but also makes the calculation of similarity more efficient.
% To address these challenges, an efficient similarity search approach based on the Monte Carlo method is introduced. It assumes two surfers starting from nodes $u$ and $v$, randomly choosing and moving to adjacent nodes. The probability of the two nodes meeting for the first time is then considered their similarity. We define the trajectory of their encounter as an instance path.

\noindent\textbf{Definition 3} A tour of a meta-path. A tour of a meta-path $A_1\cdots A_{l-1}A_lA_{l-1}\cdots A_1$, is defined as $t = \langle w_1\cdots w_{l}w_{l}\cdots w_{\text{-}1} \rangle$ in the HIN through the meta-path $P$, if it connects node $u$ and $v$ with $u = w_1 and v = w_{\text{-}1}$ $\phi(u) = \phi(v) = A_i$, satisfying the conditions: for each node $i = 1, \cdots, l, \phi(w_i) = \phi(w_{\text{-}i}) = A_i$, and $(w_1, w_2), \cdots, (w_{l\text{-}1}, w_l), (w_{\text{-}l}, w_{\text{-}l+1}), \cdots, (w_{\text{-}2}, w_{\text{-}1})$ are all edges in $G$. 

% Then we give the definition of structure similarity along a specific instance path $p$ between $u$ and $v$ based on the above discussion:
% \begin{equation} 
%     \label{(similarityOverSemanticsPath)}
%     s_p(u,v)=\prod\limits_{i=1}^{l}\beta(u) \beta(v)
% \end{equation}
% where $l$ is the length of $p$. 
% , and we only consider symmetric meta-paths. Extensions are possible but are discussed elsewhere.
Here, we assume the edge type between node types $(A_i, A_{i+1})$ is unique. Given an tour $t$ connecting node $u$ and node $v$, in this case $w_1 = u$ and $w_{\text{-}1} = v$, we get the tour similarity $Pr(t)$ of two nodes $u \in V$ and $v \in V$ as follows:
\begin{equation} 
    \label{first_meeting_probability}
    Pr(t)=
    \left\{
    \begin{array}{llr}
    1, & \text{if } l = 0; & \\
    \prod\limits_{i \in \{1, \cdots, l-1, \text{-}l+1, \cdots, \text{-}1\}}\frac{c \cdot \chi(w_i) \cdot \alpha(w_i) \cdot \mu_{R_i}}{\beta_R(w_i)}, & \text{otherwise.}
    \end{array}
    \right.
\end{equation}
where $c$ is the decay factor adopted from SimRank, $l$ is the length of the instance path, $t = (w_1\cdots w_{l}\cdots w_{\text{-}1})$ is the tour, and $R_i$ is the relation between $(w_i, w_{i+1})$. Then the calculation of the node similarity $s(u, v)$ in Eq.~\ref{power_method} can be simplified as the aggregation of all tour similarity, each of which can be calculated by Eq.~\ref{first_meeting_probability}:

% \begin{equation}
%     \label{(pathSimilarity)}
%     s_{p}(u,v)=c\prod\limits_{i=1}^{l}\alpha_i \cdot \lambda_i \cdot \mu_i \cdot \beta_i
% \end{equation}

\begin{equation}
    \label{pathSimilarity}
    s(u,v)=\sum\limits_{P\in \mathcal{P}}\sum\limits_{t\in P}Pr(t)
\end{equation}
where $\mathcal{P}$ is the set of meta-path user defined, $P$ is one meta-path in $\mathcal{P}$, and $t$ is a tour of $u$ and $v$ along the meta-path $P$.

Algorithm~\ref{alg:path_enumeration_computation} outlines the pseudocode for random surfer-pairs model of HetFS. Line 1 generates random walks for node $u$ and $v$ within $l$ steps, respectively. Their meeting probability is then added to the final similarity if the random surfers encounter each other, as depicted in Line 9.
% For each node $u \in V$, if the accessed node during every iteration is equal to node $u$, the HetFS score is assigned as 1. Otherwise, the similarity score is computed using vector multiplication as described in Eq.10. 
% The time complexity of this computation is $O(ln^2)$, and storing the similarity matrix requires $O(n^2)$ space.

% Algorithm 2 provides the pseudocode for HetFS computation. If $u = v$, their similarity is set to 1, as shown in Line 4. From Line 5 to Line 13, HetFS iteratively calculates the similarity between the two nodes. When the iteration count exceeds the predefined threshold $l$, indicating that the nodes are no longer potentially similar, Line 13 sets their similarity to 0.
% For each node $u \in V$, if the accessed node during every iteration is equal to node $u$, the HetFS score is assigned as 1. Otherwise, the similarity score is computed using vector multiplication as described in Eq.10. 
% The time complexity of this computation is $O(ln^2)$, and storing the similarity matrix requires $O(n^2)$ space.

\begin{algorithm}
\caption{HetFS Computation with random surfer-pairs model}\label{alg2}
\label{alg:path_enumeration_computation}
\begin{algorithmic}[1]
\REQUIRE heterogeneous graph $G(V, E)$, number of iterations $l$, $\chi(u)$, $\alpha(u)$, $\beta_R(u)$, $\mu_R$, node $u$, $v$
\ENSURE similarity $s$
\STATE generate $k$ tours $t_u = \langle w_1w_2\cdots w_l \rangle $, $t_v = \langle w_{\text{-}1} w_{\text{-}2}\cdots w_{\text{-}l} \rangle $ from $u$ and $v$, respectively, within $l$ steps
\STATE $s(u,v) \gets 0$
\IF{$u = v$}
\STATE $s(u,v) = 1$
\ELSE
\FOR{each tour $t$}
\IF{$w_l = w_{-l}$}
\STATE let $t = \langle w_1w_2\cdots w_l(w_{\text{-}l})\cdots w_{\text{-}2}w_{\text{-}1} \rangle$
\STATE $s(u,v) \gets s(u,v) + Pr(t)$
\ENDIF
\ENDFOR
\ENDIF
\RETURN $s(u,v)$
\end{algorithmic}
\end{algorithm}

% {\noindent\textbf{A running example}

% To identify movies most similar to a given movie ``Terminator 2" in the ``IMDB'' dataset, due to the complexity of semantics in HINs, it is necessary to specify meta-paths for conducting the search. Potential meta-paths might include MAM, MMM, MDM, or their various combinations. Initially, we pre-compute and store certain data locally, such as the tf-idf scores between movie titles, descriptions, node centrality, edge contribution, and structure information.

\begin{figure}[h]
\centering\includegraphics[width=0.95\textwidth]{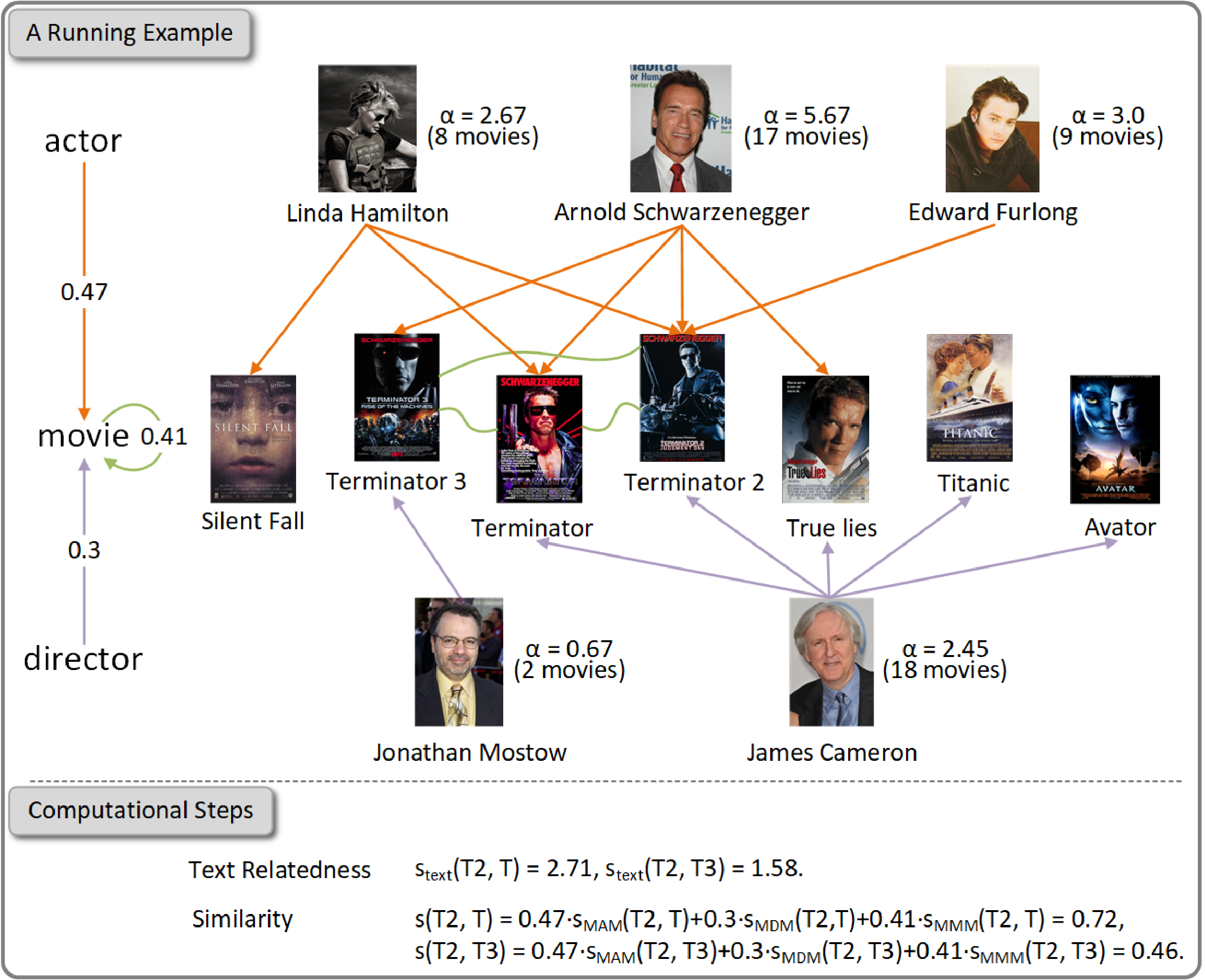}
\caption{A running example of movie network.} \label{fig:running_example}
\end{figure}

% We use a sub-network of the movie information network in our running example, shown in Fig.~\ref{fig:running_example}, to illustrate our algorithm. Suppose a user wants to find movies most similar to a given movie ``Terminator 2" from the ``IMDB" dataset. This user query corresponds to a combination of the MAM, MMM, and MDM meta-paths. To retrieve top-$k$ most similar movies to ``Terminator 2", similarities between the specific movie and other movies should be calculated first. There are two approaches for computing these similarities in HetFS. One utilizes the power method, a straightforward brute-force computation detailed in Eq.~\ref{power_method}. The other approach employs a random-surfer-pairs model, leveraging a path enumeration strategy by Eq.~\ref{first_meeting_probability} and Eq.~\ref{pathSimilarity}.

We illustrate our algorithm using a sub-network from the ``IMDB" dataset, aimed at identifying movies similar to ``Terminator 2" (T2) via the MAM, MMM, and MDM meta-paths, as shown in Fig.~\ref{fig:running_example}. There are two methods for computing similarities: the power method (Eq.~\ref{power_method}) and a random-surfer-pairs model (Eq.~\ref{pathSimilarity}), both necessitating the preprocessing of content relatedness involving movie data and initial online computations.

% The power method calculates similarities between each pair of movie nodes, while the path enumeration method focuses specifically on similarities involving ``Terminator 2" and other movies. Both methods require the pre-processing of content scores from disk, which includes data on movie titles, descriptions, node centrality, edge contributions, and structural information, which are then loaded online for initial computations.
% Each iteration calculates the similarity scores between the specified node and all other nodes, storing these in a matrix according to Eq.~\ref{power_method}. This process continues through iterations until the last meta-path is processed, which incurs significant time and space costs.

% This method factors in the top three actors and director of each movie, calculating text-related similarities and node centrality. For example, the similarity scores between "Terminator 2" and other movies in the series like "Terminator" and "Terminator 3" are computed using actor centrality, director contributions, and movie series connections.

The power method computes similarities across all movie pairs, while the path enumeration method specifically assesses similarities involving T2. Both methods require the pre-processing and loading data such as content relatedness, node centrality, edge contributions, and structural weight. Here we focus on explaining the path enumerating approach. The text relatedness between T and T2 is 2.71, and between T and T3 is 1.58. Schwarzenegger starred in 17 movies, resulting in a centrality score of 5.67, Hamilton in 8 movies with a score of 2.67, and Furlong in 9 movies with a score of 3.0. Directors James Cameron and Jonathan Mostow have centrality scores of 2.45 and 0.67, respectively. Edges contributions of starred-in, directed-by, and movie series are 0.47, 0.30, and 0.41, respectively. Then, the similarity is calculated as follows:
\begin{itemize}
    \item $s(\text{T2, T}) = 0.47\cdot s_{\text{MAM}}(\text{T2, T})+0.3\cdot s_{\text{MDM}}(\text{T2, T})+0.41\cdot s_{\text{MMM}}(\text{T2, T}) = 0.47*\frac{5.67+2.67}{3*3})+0.3*(\frac{6}{17*17})+0.41*(\frac{2.71}{2*2}) = 0.72$.
    \item $s(\text{T2, T3}) = 0.47\cdot s_{\text{MAM}}(\text{T2, T3})+0.3\cdot s_{\text{MDM}}(\text{T2, T3})+0.41\cdot s_{\text{MMM}}(\text{T2, T3}) = (0.47*\frac{5.67}{3*3})+0.3*0+0.41*(\frac{1.58}{2*2}) = 0.46$.
    \item  Analogously, we have $s(\text{T2, True Lies}) = 0.30$, $s(\text{T2, Titanic}) = 0.01$.
\end{itemize}

HetFS captures deeper content knowledge and integrates them into similarity calculations, enhancing accuracy over traditional methods that may overlook text-related elements. By utilizing advanced techniques, HetFS swiftly processes ad-hoc meta-paths, improving responsiveness and scalability in real-world applications, thus outperforming conventional path-based and graph neural network methods.

% Regarding the processing of the MMM meta-path, which reflects the similarity based on movie series from titles and descriptions, traditional path-based methods either ignore the similarity brought by text or uniformly compute topological similarities by treating simple text segmentation as nodes. Such methods are only suitable for handling straightforward texts like abstracts or keywords. In contrast, HetFS captures deeper textual knowledge, thereby calculating more precise tour similarities for each meta-path.

% By employing sophisticated, multi-faceted multi-facet techniques, HetFS can swiftly respond to user queries with ad-hoc meta-paths online. This responsiveness, along with its powerful handling of complex semantics, greatly boosts the practical utility and scalability of HetFS in real-world applications, outperforming traditional path-based or graph neural network methods.

% Furthermore, our method not only computes targeted similarity for ad-hoc meta-paths but also extracts deep insights from existing information in the graph. 

\subsection{Analysis}

% In this section, we analyze the theoretical property of the HetFS algorithm. 
% , focusing on its efficiency and accuracy in handling complex queries
% \noindent\textbf{Time Complexity for worst-case graphs}
% \noindent\textbf{Time Complexity}

% TSF+panther

The similarity search in HetFS, as detailed in Algorithm~\ref{alg2}, is highly efficient. Theorem~\ref{thm1} demonstrates that HetFS can estimate a similarity score between a specific node and other involving nodes in $\mathcal{O}(\overline{d}^lm + \log k)$ time. This performance is independent of the graph size, showcasing the scalability of HetFS.
% independent of the size of the graph.

\begin{theorem}\label{thm1}
In HetFS, the average time required to compute a similarity score between a specific node and other involving nodes in a user query with ad-hoc meta-paths is bounded by $\mathcal{O}(\overline{d}^lm)$. The total time for processing the query is bounded by $\mathcal{O}(\overline{d}^lm + \log k)$.

% Each iteration in tour runs $\mathcal{O}$ time, so a tour $t_u$ of length $l$ runs in $\mathcal{O}(l\cdot n)$ time. Suming up for $k$ tours follows that the expected running time of $\mathcal{O}(\frac{m}{\varepsilon^2}\log{\frac{n}{\delta}})$.    
\end{theorem}

\begin{proof}
The estimation of top-$k$ similarity score involves two main parts.

Calculation Phase. To calculate $Pr(t)$ for a single tour $t$, HetFS begins by traversing over each tour within $l$ steps. We notice that in each iteration of the random walk, each edge in the graph is traversed at most once. Thus, the time complexity of each path is $\mathcal{O}(m)$. Following this, it aggregates tours across various meta-paths. Considering the average degree in the graph as $\overline{d}$, the average number of meta-path as $\lvert\overline{\mathcal{P}}\rvert$, and the average number of paths in each meta-path as $\lvert\overline{P}\rvert$, the overall average number of paths $\lvert\overline{P}\rvert\lvert\overline{\mathcal{P}}\rvert$ is proportional to $\overline{d}^l$. Consequently, the calculating cost is $\mathcal{O}(\overline{d}^lm)$.

Search Phase. The process of retrieving top-$k$ similar nodes based on the HetFS similarity scores introduces an additional time cost of $\mathcal{O}(\log k)$.

In summary, the overall time complexity for retrieving top-$k$ similar nodes is $\mathcal{O}(\overline{d}^lm + \log k)$, confirming the efficiency of the algorithm.

% \noindent\textbf{Space Complexity}

% The space complexity for storing paths

 % $\sum_{i=1}^(i=l)\overline{d}^i$.

\end{proof}

\section{Experiments}\label{sec5}

Extensive experiments are conducted to address the following research questions. This section reports our results.
\begin{itemize}
    \item \textbf{RQ1}. How does HetFS perform compared to state-of-the-art methods for user queries with ad-hoc meta-paths in terms of both effectiveness and efficiency?
    \item \textbf{RQ2}. How does HetFS perform compared to state-of-the-art methods across various downstream graph mining tasks, including link prediction, node clustering, and node classification?
    \item \textbf{RQ3}. What is the impact of different components, such as node information, edge information, and content information, on the performance of HetFS?
    % \item \textbf{RQ4}. How do various hyper-parameters affect the performance of HetFS?
\end{itemize}

\subsection{Experimental Design}

We describe the details of datasets, methods, how to set the ground truth and parameters, and how to assess the effectiveness of each method in our experiments.

\subsubsection{Data Description} 

We use 3 real-world datasets described as follows to evaluate the performance of HetFS as compared to state-of-the-art baselines.
\begin{itemize}
\item DBLP $\footnote{https://aminer.org/data}$: is an academic dataset \cite{tang2008arnetminer} extracted from website sources. We adopt a subset of Academic that includes information about paper writing, citations, publications, titles, and abstracts, covering three types of nodes and other related content. Authors are categorized into four research areas: DB, DM, NLP, and CV. Specifically, we extracted the titles and abstracts of all papers and aggregated them into a title corpus and an abstract corpus. Using these two corpora as references, HetFS calculated the textual similarity between each pair of papers.
% \item DBLP $\footnote{https://web.cs.ucla.edu/~yzsun/data/}$: a bibliographic website for computer science. We adopt a subset of DBLP that includes information about paper writing, citations, publications, and titles, encompassing four types of nodes. Authors are categorized into four research areas: Database, Data Mining, Artificial Intelligence, and Information Retrieval.
% \item ACM $\footnote{https://github.com/Jhy1993/HAN/tree/master/data}$: an academic citation network that extracts papers published in conferences such as KDD, SIGMOD, SIGCOMM, MobiCOMM, and VLDB. It categorizes the papers into three classes according to their published conference: Database, Wireless Communication, and Data Mining.
\item IMDB $\footnote{https://grouplens.org/datasets/hetrec-2011/}$: is a dataset that encompasses information about both movies and TV shows. It includes details such as cast, production crew, plot summaries, and related information, covering four types of nodes and other related content. The movies are categorized into five classes based on their genre information: Action, Comedy, Drama, Romance, and Thriller.
\item LastFM $\footnote{https://grouplens.org/datasets/hetrec-2011/}$: is an online music website containing data on music listening behavior. It comprises information about users, artists, and tags, along with their interactions. We use the dataset released by MAGNN \cite{fu2020magnn} for the link prediction task.
\end{itemize}
Table 3 summarizes the dataset statistics.

\begin{table}[h]
\fontsize{12}{12}\selectfont
\caption{\centering Statistics of datasets.}\label{statisticOfDatasets}
% \scriptsize
\footnotesize
% \resizebox{0.49\textwidth}{!}{
\begin{tabular}{cccccccccm{5cm}}
\toprule
Dataset & \multicolumn{2}{c}{Entity} & \multicolumn{2}{c}{Edge} \\
\midrule
\multirow{2}*{Academic} & \textbf{A}uthor: 28645 & \textbf{P}aper: 21044 & A-P: 69311 & P-P: 34238 \\
 & \textbf{T}erm: 22551 & \textbf{V}enue: 18 & P-T: 171774 & V-P: 21044 \\
\midrule
% \multirow{2}*{DBLP} & Author: 4057 & Paper: 14328 & writing: 39290 & mentioning: 171620 \\
%  & Term: 7723 & Venue: 20 & publishing: 28656 \\
% \midrule
% \multirow{2}*{ACM} & Author: 5959 & Paper: 3025 & writing: 100797 & citing: 165904 \\
%  & Subject: 56 & Term: 1902 & publishing: 42872 & mentioning: 284156 \\
% \midrule
\multirow{3}*{IMDB} & \textbf{M}ovie: 8742 & \textbf{D}irector: 3279 & D-M: 6582 & A-M: 19612 \\
 & \textbf{A}ctor: 9418 & \textbf{T}erm: 7295 & M-T: 16247 & M-G: 4914 \\
 & \textbf{G}enre: 17 & & & \\
\midrule
\multirow{2}*{LastFM} & \textbf{U}ser: 1892 & \textbf{A}rtist: 17632 & U-U: 12717 & U-A: 92834 \\
 & \textbf{T}ag: 1088 & & A-T: 23253 & \\
\bottomrule
\end{tabular}
% }
\end{table}

\subsubsection{Baseline} 

We compare HetFS against various methods, including traditional graph similarity search methods, graph embedding models, and GNNs. The baseline methods are outlined below.

\begin{itemize}
\item HGT \cite{hu2020heterogeneous} is a heterogeneous GNN that extends the transformer architecture to the graph-structured data. It utilizes the self-attention mechanism of the transformer to incorporate the unique characteristics of HINs. 
\item HetGNN \cite{zhang2019heterogeneous} is a heterogeneous GNN with the ability to capture both structural and content-related information from heterogeneous graphs. It utilizes attention mechanisms to aggregate information from neighboring nodes.
\item MHGNN \cite{li2023metapath} is a meta-path-based heterogeneous GNN with the ability to capture both structural and content-related information from heterogeneous graphs. 
\item MAGNN \cite{fu2020magnn} is a heterogeneous GNN designed to capture information from various levels by incorporating node attributes, intermediate semantic nodes, and multiple meta-paths. This enables MAGNN to leverage both network topology and content-related information for effective representation learning in HINs.
\item HAN \cite{wang2019heterogeneous} is a heterogeneous model that proposes an attention mechanism to address the heterogeneity inherent in nodes and edges. Specifically, it employs two layers of attention structures to capture weights at both the node and semantic levels intricately.
\item RGCN \cite{schlichtkrull2018modeling} is an extension of GCN for relational (multiple edge types) graphs, which can be conceptualized as a weighted combination of standard graph convolutions, each tailored to a specific edge type within the graph.
\item DeepWalk \cite{perozzi2014deepwalk} is a homogeneous graph embedding method that pioneered the use of random walks to transform the graph structure into sequences, subsequently utilizing these sequences to embed nodes of the graph.
% \item Node2Vec \cite{grover2016node2vec} is an enhanced homogeneous model derived from DeepWalk. While DeepWalk essentially utilizes a depth-first traversal algorithm with the ability to revisit visited nodes to generate node sequences, Node2Vec adopts a novel random walk strategy. It combines breadth-first and depth-first traversals by introducing parameters to balance local and global similarity.
% \item Metapath2Vec \cite{dong2017metapath2vec} is a heterogeneous model that generates random walk sequences based on pre-defined meta-paths to construct heterogeneous neighbors for each node.
\item PathSim \cite{sun2011pathsim} is a traditional heterogeneous similarity measure calculated by the path instance count along a specific meta-path. It addresses the multi-edge similarity search problem between two nodes.
\end{itemize}

\subsubsection{Experimental Setup}

When dealing with meta-path-free scenarios, PathSim aggregates the results from all possible paths and then merges them and sorts them to obtain the final ranking. Specifically, we set $\epsilon$ to 0.000005 and $k$ to 1000, that is, we consider results with similarity values above this threshold  $\epsilon$ and rank $k$ within the top 1000 nodes as similar nodes. Since DeepWalk handles meta-path-free scenarios, it does not distinguish between node types and conducts random walks on the entire graph to generate preprocessed paths. For the HetGNN algorithm, we adopt the experimental parameters from the original paper.

All experiments are conducted on a Linux machine featuring an Intel(R) Xeon(R) Silver 4208 CPU running at 2.1GHz and 128GB of memory.

\subsection{Ad-hoc Queries (RQ1)}

As mentioned earlier, various meta-paths convey varied semantic meanings, leading to diverse query results. We focus on the differences in effectiveness and efficiency of the algorithm under ad-hoc query scenarios.

First, a case study is conducted on the ``IMDB" dataset. We use the movie ``Terminator 2: Judgment Day", a sci-fi film directed by James Cameron and starring Arnold Schwarzenegger, among others, as an example and query the top 8 movies that are most similar to it. The experiment compares outcomes for three scenarios: two meta-paths (``MDM", i.e., movie-director-movie, ``MAM," i.e., movie-actor-movie) and meta-path-free. These scenarios represent co-actors, movies directed by the same director, and a comprehensive search for similar movies using all available information. Graph embedding methods and HGNNs rank nodes based on their preference score, computed as the inner product between their embeddings.

As previously highlighted, diverse meta-paths encapsulate unique semantic meanings, reflecting varied user interests in information. This underscores our emphasis on the disparate outcomes yielded by searches conducted through different paths. Table~\ref{tab:result_imdb} displays the similarity results obtained under different specific meta-paths, revealing notable variations. The experimental results are consistent with the expectations. Movies retrieved based on the ``MAM" meta-path have similar leading actors to ``Terminator 2", while movies retrieved based on the ``MDM" meta-path have similar directors to ``Terminator 2".

% Table generated by Excel2LaTeX from sheet 'merge'
\begin{table*}[htbp]
\centering
\caption{\centering Case study on query ``Terminator 2" under different meta-paths on ``IMDB" dataset.}
\footnotesize
\renewcommand\arraystretch{1.2}
\resizebox{\textwidth}{!}{
\begin{tabular}{|c|ll|ll|}
\multicolumn{1}{c}{\multirow{1}[0]{*}{}} & \multicolumn{2}{c}{(a) PathSim} & \multicolumn{2}{c}{(b) DeepWalk} \\
\hline
Rank  & meta-path: MAM   & meta-path: MDM   & meta-path: MAM   & meta-path: MDM \\
\hline
1     & True Lies & True Lies & Hercules in New York & True Lies \\
2     & Eraser & The Abyss & Pumping Iron & The Abyss \\
3     & Jingle All the Way & Aliens & Detroit Rock City & Aliens \\
4     & Conan the Barbarian & Titanic & King Kong Lives & Titanic \\
5     & Total Recall & Ghosts of the Abyss & The Kid \& I & Aliens of the Deep \\
6     & End of Days & Aliens of the Deep & Kindergarten Cop & Ghosts of the Abyss \\
7     & The Running Man & -     & Brainscan & My Name Is Bruce \\
8     & Red Heat & -     & Children of the Corn & On\_Line \\
% 9     & Twins & -     & Jingle All the Way & Passion of Mind \\
% 10    & Collateral Damage & -     & American Heart & Ma vie en rose \\
\hline
\multicolumn{1}{c}{\multirow{1}[0]{*}{}} & \multicolumn{2}{c}{(c) HetGNN} & \multicolumn{2}{c}{(d) HetFS} \\
\hline
Rank  & meta-path: MAM   & meta-path: MDM   & meta-path: MAM   & meta-path: MDM \\
\hline
1     & Jingle All the Way & Titanic & Terminator & True Lies \\
2     & End of Days & Ghosts of the Abyss & True Lies & Aliens \\
3     & Total Recall & The Abyss & Total Recall & The Abyss \\
4     & American History X & Aliens of the Deep & Jingle All the Way & Titanic \\
5     & Terminator 3 & True Lies & End of Days & Ghosts of the Abyss \\
6     & Pumping Iron & Aliens & The Running Man & Aliens of the Deep \\
7     & Kindergarten Cop & El Callejon de los Milagros & Red Heat & - \\
8     & Hercules in New York & Smokey and the Bandit Part 3 & Pumping Iron & - \\
% 9     & Twins & Care Bears Movie II: A New Generation & Eraser & - \\
% 10    & True Lies & Over the Hedge & King Kong Lives & - \\
\hline
\end{tabular}%
}
\label{tab:result_imdb}%
\end{table*}%

Next, we enumerate all paths with a length of 2, treating them as the complete semantics, to conduct meta-path-free experiments. The outcomes are listed in Table~\ref{tab:result_imdb_meta-path-free}. Our method, HetFS, demonstrates the ability to retrieve movies associated with the provided movie effectively. Specifically, HetFS not only identifies movies with the same main actor or director but also discovers series such as ``Terminator 3" and ``Terminator." The former appears only in the top ten search results of PathSim among the other three methods, while the latter does not appear in the top ten search results of any other methods. Series are generally considered to have high relevance. 
% Specifically, ``Terminator 3" shares the same main actor, Arnold Schwarzenegger, with ``Terminator 2", while ``Terminator" and ``Terminator 2" share no common actors or directors besides being part of the same series. 

% For convenience, we have counted the number of results each method found in the user study and displayed them in the last row of the table.

\begin{table*}[h]
\centering
\caption{\centering Case study on query ``Terminator 2" on ``IMDB" dataset under meta-path-free scenarios.}
\footnotesize
\renewcommand\arraystretch{1.2}
\resizebox{\textwidth}{!}{
\begin{tabular}{|c|l|l|l|l|}
% \multicolumn{1}{c}{\multirow{1}[0]{*}{}} &  \multicolumn{1}{c}{(a) PathSim} &  \multicolumn{1}{c}{(b) DeepWalk} &  \multicolumn{1}{c}{(c) HetGNN} &  \multicolumn{1}{c}{(d) HetFS} \\
\hline
Rank & PathSim & DeepWalk & HetGNN & HetFS \\
\hline
1 & The Abyss & True Lies & True Lies & Terminator \\
2 & Aliens & Ghosts of the Abyss & Aliens of the Deep & Terminator 3 \\
3 & Titanic & Aliens of the Deep & Pumping Iron & True Lies \\
4 & Ghosts of the Abyss & Detroit Rock City & Ghosts of the Abyss & Titanic \\
5 & Aliens of the Deep & Brainscan & Hercules in New York & The Abyss \\
% 6 & True Lies & King Kong Lives & Aliens & Aliens \\
% 7 & Terminator 3 & Children of the Corn & Red Heat & Hercules in New York \\
% 8 & Terminator Salvation & Hercules in New York & The Abyss & Pumping Iron \\
% 9 & Eraser & Pumping Iron & Kindergarten Cop & Collateral Damage \\
% 10 & Jingle All the Way & The Kid \& I & Collateral Damage & Jingle All the Way \\
\hline
\end{tabular}%
}
\label{tab:result_imdb_meta-path-free}%
\end{table*}%

Next, We utilize the ``DBLP" dataset and retrieve the top 8 authors most similar to Jiawei Han, who has the highest number of paper records in the dataset (154 papers). The experiment compares results for three scenarios: two meta-paths and meta-path-free, representing co-authors, co-conference participants, and a comprehensive search for similar authors using all available information. 
 % (``APA", i.e., author-paper-author, ``APVPA", i.e., author-paper-venue-paper-author)

The outcomes are displayed in Table~\ref{tab:result_dblp}. Notably, Yizhou Sun was a student of Jiawei Han, and now being an influential mentor in the field of graph data mining. They are close collaborators with strong connections. HetFS and DeepWalk rank Yizhou Sun at the top, while HetGNN is in the third position, and PathSim is in the sixth position. However, the results of HetGNN are not satisfactory. For instance, author Philip Yu did not appear in the top ten results of other algorithms but was considered by HetGNN as the most similar author, ranked first.

\begin{table*}[h]
% \fontsize{12}{12}\selectfont \cellcolor{gray!40}
\caption{\centering Case study on query Jiawei Han under different meta-paths on ``DBLP" dataset.}\label{tab:result_dblp}
\scriptsize
% \footnotesize
\renewcommand\arraystretch{1.3}
\resizebox{\textwidth}{!}{
\begin{tabular}{|c|ccc|ccc|m{2cm}}
\multicolumn{1}{c}{\multirow{1}[0]{*}{}} & \multicolumn{3}{c}{\textbf{(a) PathSim}} & \multicolumn{3}{c}{\textbf{(b) DeepWalk}} \\
\hline
\textbf{Rank} & meta-path: APA & meta-path: AVA & meta-path-free & meta-path: APA & meta-path: AVA & meta-path-free \\
\hline
1 & \textbf{Yizhou Sun} & Philip S. Yu & Philip S. Yu & Xiaolei Li & Haixun Wang & \textbf{Yizhou Sun} \\
2 & Xifeng Yan & Christos Faloutsos & Christos Faloutsos & Tim Weninger & Michalis Vazirgiannis & Xiao Yu \\
3 & Chi Wang & Hui Xiong & Divesh Srivastava & Brandon Norick & Kevin Chen-Chuan & Marina Danilevsky \\
4 & Xiao Yu & Wei Wang & Xifeng Yan & \textbf{Yizhou Sun} & Min Wang & Bo Zhao \\
5 & Jing Gao & Tao Li & Gerhard Weikum & George Brova & C. Lee Giles & Yintao Yu \\
6 & Bo Zhao & Jian Pei & \textbf{Yizhou Sun} & Marina Danilevsky & Francesco Bonchi & Tim Weninger \\
7 & Dong Xin & Jie Tang & Aristides Gionis & Yintao Yu & Jianyong Wang & Zhijun Yin \\
8 & Bolin Ding & Haixun Wang & Wei Wang & Bo Zhao & Enhong Chen & Fangbo Tao \\
9 & Deng Cai & Lei Chen & Francesco Bonchi & Tianyi Wu & HweeHwa Pang & Jialu Liu \\
10 & Xiaofei He & Wei Fan & Tao Li & Roland Kays & Aixin Sun & Xiaolei Li \\
\hline
\multicolumn{1}{c}{\multirow{1}[0]{*}{}} & \multicolumn{3}{c}{\textbf{(c) HetGNN}} & \multicolumn{3}{c}{\textbf{(d) HetFS}} \\
\hline
\textbf{Rank} & meta-path: APA & meta-path: AVA & meta-path-free & meta-path: APA & meta-path: AVA & meta-path-free \\
\hline
1 & Jimeng Sun & Rajeev Rastogi & Philip Yu & Quanquan Gu & Philip S. Yu & Quanquan Gu\\
2 & Brandon Norick & Tamraparni Dasu & Chen Chen & \textbf{Yizhou Sun} & Christos Faloutsos &  \textbf{Yizhou Sun} \\
3 & Guojie Song & Guiquan Liu & \textbf{Yizhou Sun} & Dong Xin & Divesh Srivastava & Dong Xin \\
4 & Zeng Lian & Dashun Wang & Xifeng Yan & Deng Cai & Wei Wang & Deng Cai \\
5 & Marina Danilevsky & Xiao Yu & Bolin Ding & Chi Wang & Yongxin Tong & Chi Wang \\
6 & \textbf{Yizhou Sun} & Ming Li & Xiao Yu & Tim Weninger & Lei Chen & Tim Weninger \\
7 & Christos Boutsidis & Neal E. Young & Peixiang Zhao & Philip S. Yu & Haixun Wang & Philip S. Yu \\
8 & Jinyan Li & Jian Xu & Spiros Papadimitriou & Xiaofei He & Tao Li & Xiaofei He \\
9 & Yang Cao & Robson Leonardo & Sangkyum Kim & Xiao Yu & Jian Pei & Xiao Yu \\
10 & Yu-Ru Lin & Chengnian Sun & Hong Cheng & Jing Gao & Hui Xiong & Jing Gao \\
\hline
\end{tabular}
}
\end{table*}

In order to meet the demands of ad-hoc queries from users, the algorithm's responsiveness in switching meta-paths is crucial. We then recorded the CPU time required by different methods to search for similar nodes based on different paths. Fig.~\ref{fig:time_cost_compare} illustrates the efficiency performance of the compared methods. A logarithmic scale is used in the graph to visualize the differences in results better. The graph embedding method and HGNNs require a significant amount of time to train node embeddings. In the meta-path-free scenario, the most time-consuming algorithm (HetGNN) takes 63,038,114 ms (approximately 105 min) and 5,281,670 ms (approximately 88 min) to process queries on the ``DBLP'' dataset and ``IMDB'' dataset, respectively. Such long waiting times are almost intolerable for ad-hoc searches. On the other hand, PathSim and HetFS do not require time-consuming training, and since they support single-source queries, their average time costs for processing a query of the ``DBLP'' dataset are only 6 ms and 8 ms, respectively. 

\begin{figure}[h]
    \centering
        \includegraphics[width=0.9\textwidth]{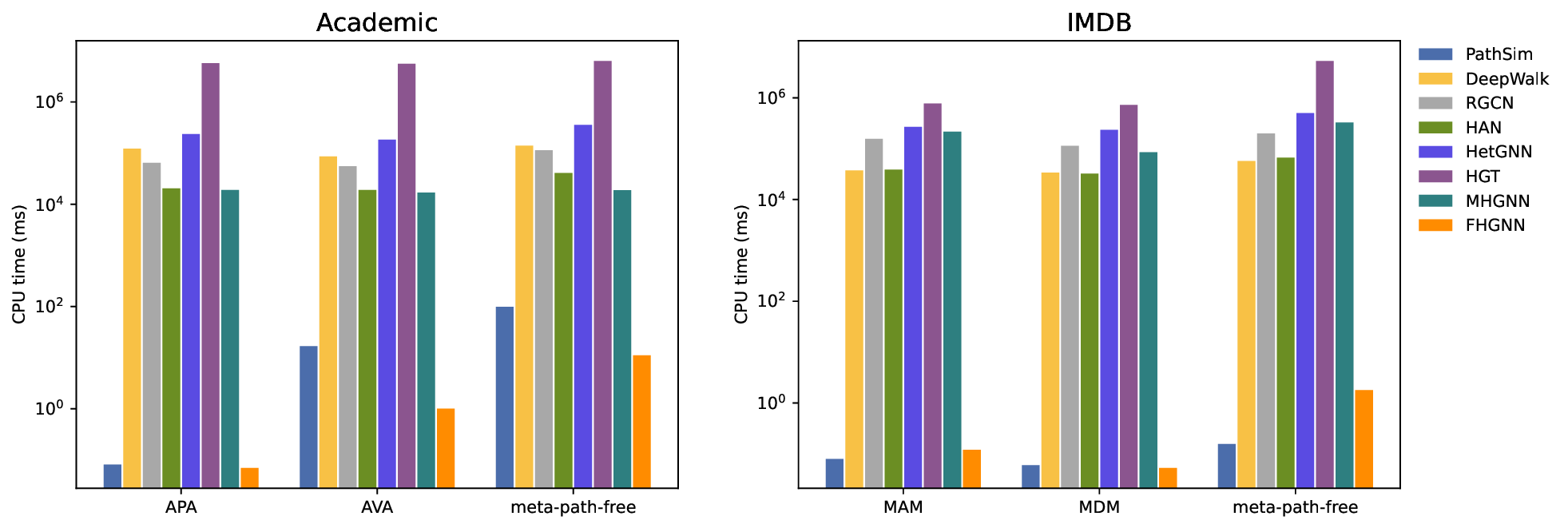}
    \caption{The time cost under different meta-paths on ``DBLP'' and ``IMDB'' datasets.} \label{fig:time_cost_compare}
\end{figure}

\subsection{Applications (RQ2)}

To validate the practicality of HetFS, we compare it with measures across various tasks.

\subsubsection{Link Prediction}

% Given a time $Ts$, prior to the split year ($Ts$), data is assigned as training data, while data post-Ts is utilized as test data. We set $Ts$ to 2014 in the ``DBLP" dataset. We conduct two types of tasks, collaboration link prediction and citing link prediction. In the former, we remove author-paper edges post-2014 and retrieve the top-k similar authors of author $u$ based on existing links. These $k$ candidate authors are then regarded as having a collaboration link with author $u$. Similarly, in the latter task, we eliminate paper-paper citing edges after 2014 and retrieve the top-k similar papers for each paper based on existing links. Subsequently, we consider citing links to exist between these $k$ candidate papers and the authors of paper $u$. $AUC$ and $F1$ are utilized for quality measurement, where $F1$ = $\frac{2*recall*precision}{recall+precision}$, with $TP$ = $|candidates\cap deleted \text{ } edges|$, $TP+FP$ = $|deleted \text{ } edges|$, $TP+FN = |candidates|$, $precision = \frac{TP}{TP+FP}$, $recall = \frac{TP}{TP+FN}$. Moreover, only new links among nodes in the test data are considered, and duplicated links are removed during evaluation. 

% We conduct two types of tasks: collaboration link prediction and citing link prediction. In the former. For the latter task, we eliminate paper-paper citing edges after 2014 and retrieve the top-k similar papers for each paper based on existing links. Subsequently, we consider citing links to exist between these $k$ candidate papers and the authors of paper $u$.

Two split settings are employed for the link prediction task. For the ``DBLP'' dataset, given a time $Ts$, prior to the split year ($Ts$), data is assigned as training data, while data post-Ts is utilized as test data. We set $Ts$ to 2014 in the ``DBLP" dataset. We remove author-paper edges post-2014 and retrieve the top-k similar authors of author $u$ based on existing links. These $k$ candidate authors are then regarded as having a collaboration link with author $u$. For the ``LastFM'' dataset, we randomly sample a subset of links for training and reserve the remaining links for evaluation, maintaining a train/test ratio of 7:3. $AUC$ and $F1$ are utilized for quality measurement. Moreover, during evaluation, only new links between nodes in the test data are considered, and any duplicated links are removed.

% In the ``IMDB" dataset, we randomly sample a portion of links for training and utilize the remaining links for evaluation, with a train/test ratio (in terms of movie numbers) set to 7:3. 

Table~\ref{linkprediction} presents the performance metrics for all models, highlighting superior results in bold. It is observed that HetFS exhibits relatively favorable outcomes compared to the graph embedding method and HGNNs, indicating that the systematic integration of path similarity and content similarity effectively captures graph information for link prediction. We noticed that the results of PathSim are not very satisfactory. This is because PathSim tends to treat authors with comparable connections to a common intermediate node as similar. This is because PathSim fails to incorporate content information in the graph and its computational principles. PathSim tends to equate similarity with the consistency of path quantities between two nodes and intermediate nodes. While theoretically sound, this often leads to errors in practical applications.

% For example, authors who share a similar number of paper publications at the same conferences are identified as similar. While this conceptualization is theoretically sound, PathSim tends to consider authors who share the same number of papers presented at the same conference as identical, assigning them a maximum similarity value of 1. Similarly, if two authors collaborate across all their published works, PathSim considers them completely similar. However, instances where authors collaborate on all their works are infrequent, causing the similarity of collaboratively engaged authors to be overshadowed by those who share conference participation. Consequently, authors with collaborative ties may be ranked lower than those who share conference participation in PathSim's link prediction outcomes, resulting in less than optimal results.

% HetFS outperforms all baselines, particularly excelling in citation link prediction.

\begin{table*}[h]
\centering
% \fontsize{12}{12}\selectfont
\caption{\centering Experimental results (\%) on ``DBLP'' and ``LastFM'' datasets for the link prediction tasks. Vacant positions (``-") are due to a lack of similarity based on that method.}\label{linkprediction}
% \scriptsize
\footnotesize
\resizebox{\textwidth}{!}{
\begin{tabular}{ccccccccccc}
\toprule
          &       & \multicolumn{1}{l}{PathSim} & \multicolumn{1}{l}{DeepWalk} & \multicolumn{1}{l}{RGCN} & \multicolumn{1}{l}{MAGNN} & \multicolumn{1}{l}{HetGNN} & \multicolumn{1}{l}{MHGNN} & \multicolumn{1}{l}{HGT} & \multicolumn{1}{l}{HetFS} \\
    \midrule
    \multirow{2}[0]{*}{DBLP} & AUC   & 54.15  & 71.67  & 75.02 & -  & 76.60  & 75.64 & 75.02 & \textbf{76.91} \\
          & MRR   & 58.9  & 73.84  & \textbf{79.76} & -  & 78.48  & 78.13 & 77.94 & 78.63 \\
    \midrule
    \multirow{2}[0]{*}{LastFM} & AUC   & 43.56  & 50.49  & 55.87 & 54.66 & \textbf{59.04} & 58.43 & 54.44 & 58.29 \\
          & MRR   & 58.06  & 65.53  & 72.08 & 68.12 & 77.56 & 75.5  & 73.18 & \textbf{79.37} \\
% AUC (type-1)   & 0.541  & 0.736  & \textbf{0.760}  & 0.742  \\
% F1 (type-1)    & 0.177  & 0.650  & \textbf{0.714}  & 0.681  \\
% AUC (type-2)   & 0.600  & 0.778  & 0.785  & \textbf{0.812}  \\
% F1 (type-2)    & 0.339  & 0.741  & 0.768  & \textbf{0.810}  \\
\bottomrule

\end{tabular}
}
\end{table*}

\subsubsection{Node Clustering and Classification}

For node clustering and classification tasks, we follow the methodology outlined in HetGNN \cite{zhang2019heterogeneous}. We categorize authors in the ``DBLP" dataset into four distinct research domains: DM, CV, NLP, and DB. The top three venues for each domain are selected, and authors are labeled based on the predominant publication venues within these domains. Authors lacking publications in these venues are omitted from the evaluation. Node information is derived from the complete dataset.

To cluster node, we assumed that authors sharing similar research interests were likely to be active in similar domains. Consequently, we classified authors based on the conference categories where they exhibited the highest participation. More precisely, we record the most frequently attended conference for each author and multiply this information by the similarity value between the target author and others, which is calculated as the HetFS score in a meta-path-free manner. Then, the label preference scores were accumulated for each label, with the top-ranked label used for clustering. This method was applied to derive clustering results for HetFS and PathSim. For graph embedding methods and HGNNs, we utilized the learned node embeddings as input for a k-means clustering algorithm. Clustering performance was assessed using metrics such as normalized mutual information (NMI) and adjusted Rand index (ARI).

A methodology akin to node clustering was applied for the multi-label classification task to derive node classification outcomes for PathSim and HetFS. For graph embedding methods and HGNNs, the acquired node embeddings served as input for a logistic regression classifier. Notably, the train/test ratio was set at 1:9. Evaluation metrics encompass both Micro-F1 and Macro-F1.

Table~\ref{clustering&classification} presents the results of all methods, with the best outcomes highlighted in bold. The observations are as follows: (1) Most models exhibit excellent performance in multi-label classification, achieving high Macro-F1 and Micro-F1 scores (over 0.88 in the ``DBLP'' dataset and 0.55 in the ``IMDB'' dataset). This is reasonable given the distinct nature of authors across the four selected domains. (2) Graph embedding methods and HGNNs yield the best classification results, while HetFS also performs well. This indicates that HetFS effectively learns graph information for the node clustering task. (3) HetFS outperforms PathSim in both node clustering and node classification tasks, demonstrating that leveraging content information enhances embedding performance.

% The superior classification results of PathSim can be attributed to its consideration of authors who co-participate in conferences as highly similar. Since we also classify authors based on their most frequently attended conferences in the node clustering task, PathSim aligns well with this clustering objective.

\begin{table*}[h]
\centering
% \fontsize{12}{12}\selectfont
\caption{\centering Experimental results (\%) on ``DBLP'' and ``IMDB'' datasets for the node clustering and classification task.}\label{clustering&classification}
% \scriptsize
\footnotesize
\resizebox{\textwidth}{!}{
\begin{tabular}{ccccccccccc}
\toprule
          &       & \multicolumn{1}{l}{PathSim} & \multicolumn{1}{l}{DeepWalk} & \multicolumn{1}{l}{RGCN} & \multicolumn{1}{l}{HAN} & \multicolumn{1}{l}{MAGNN} & \multicolumn{1}{l}{HetGNN} & \multicolumn{1}{l}{MHGNN} & \multicolumn{1}{l}{HGT} & \multicolumn{1}{l}{HetFS} \\
    \midrule
    \multirow{4}[0]{*}{DBLP} & NMI   & 83.21  & 88.6  & 87.55 & 90.33 & 90.21 & 89.6  & 88.07 & \textbf{91.22} & 90.77 \\
          & ARI   & 85.29  & 91.34  & 89.49 & 90.89 & 91.3  & \textbf{93.1}  & 91.26 & 92.13 & 91.46 \\
          & macro-f1  & 86.18  & 90.45  & 89.88 & 91.67 & \textbf{92.16} & 91.74 & 91.84 & 92.1  & 92.3 \\
          & micro-f1  & 86.64  & 91.58  & 90.49 & 92.05 & 92.24 & 92.53 & \textbf{92.97} & 92.78 & 92.54 \\
    \midrule
    \multirow{4}[0]{*}{IMDB} & NMI    & 53.23  & 54.09  & 57.31 & 57.04 & 56.32 & 56.19 & 56.44 & \textbf{57.89} & 57.12 \\
          & ARI    & 55.06  & 55.34  & 59.2  & 62.77 & \textbf{63.07} & 62.96 & 61.86 & 60.43 & 61.88 \\
          & macro-f1  & 55.21  & 56.77  & 58.85 & 57.74 & 56.49 & 58.28 & 58.01 & 59.12 & \textbf{59.78} \\
          & micro-f1  & 57.76  & 58.4  & 62.05 & 64.63 & 64.67 & 64.3  & 63.94 & \textbf{65.2}  & 63.36 \\
% DBLP & \textbf{PathSim} & \textbf{DeepWalk} & \textbf{HetGNN} & \textbf{HetFS} \\
% \midrule
% NMI   & 0.832  & \textbf{0.886}  & \textbf{0.886}  & 0.820  \\
% ARI   & 0.882  & 0.919  &\textbf{ 0.921}  & 0.868  \\
% \midrule
% MacroF1 & 0.947  & \textbf{0.976}  & 0.974  & 0.945  \\
% MicroF1 & 0.951  & \textbf{0.977}  & 0.975  & 0.947  \\
\bottomrule
\end{tabular}
}
\end{table*}

\subsection{Ablation Study}

HetFS is an integrated similarity search method designed to aggregate information at different levels. How does node information impact the total performance? Does semantic information optimize information processing? Is content information effective for enhancing the performance of the path information-based method? We conducted ablation studies to evaluate the performance of various components to answer the above questions (RQ3), including (1) - node, removing node centrality from the method and treating all nodes as identical; (2) - semantics, excluding edge contributions from the method, and treating all edge weights as identical; (3) - content, removing content information from the method and relying solely on path information in the HIN for similarity search. The results of node clustering and node classification on the IMDB dataset are shown in Fig.~\ref{fig:ablation_study}.

\begin{figure}[h]
    \centering
        \includegraphics[width=\textwidth]{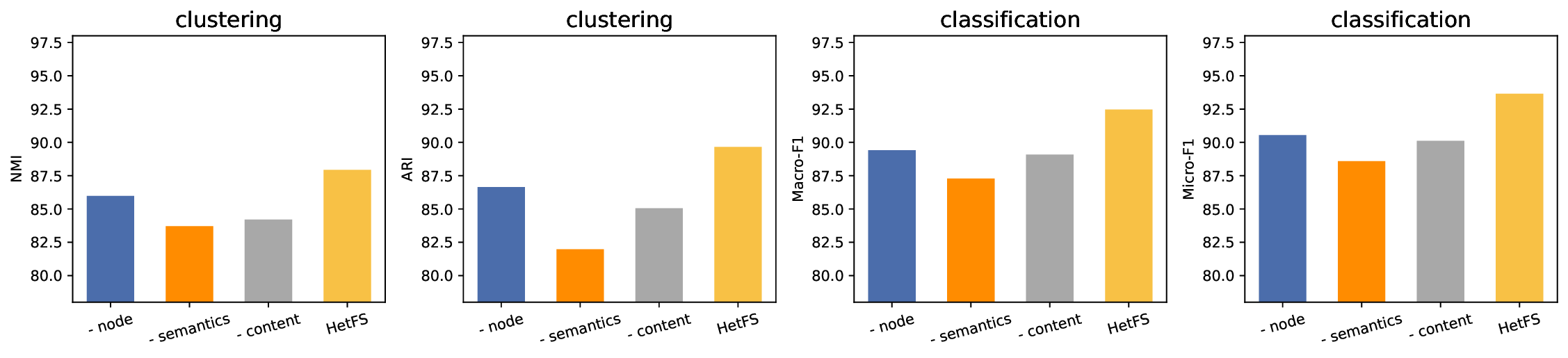}
    \caption{Performances without various components on ``IMDB" datasets.} \label{fig:ablation_study}
\end{figure}

According to Fig.~\ref{fig:ablation_study}, we can observe that: (1) HetFS performs better than all partial methods with different components removed, suggesting that node information, semantics information, and content information all contribute to the aggregation of heterogeneous information to varying extents. (2) Among all the different components, the removal of semantics information has the most pronounced impact, indicating that content provides more effective information to the method compared to other components.

All the experiments above demonstrated that a combination of path information and content information was effective and efficient in mining graph data. Ad-hoc query experiments revealed that the results obtained by searching with different meta-paths focused on different content aspects. HetFS not only delved deeply into graph data to provide meaningful results but also swiftly adapted to user-given meta-paths for answering ad-hoc queries. Downstream applications further demonstrated that, while ensuring rapid responses in on-the-fly queries, HetFS delivered results comparable to state-of-the-art methods.

\section{Conclusion}\label{sec13}

This paper presented HetFS, an efficient and effective similarity measure designed to address four limitations in existing solutions for HINs: (1) high computational overhead on ad-hoc queries with user-given meta-paths, (2) disregarding node centrality, (3) omission of edge contribution, and (4) neglecting node content features. Specifically, HetFS leveraged content, node, edge, and structural information to aggregate data from HINs. Content information, integrated into path information later, underwent transformation via type-specific functions to unify it into a latent space. Path information incorporated node centrality and edge contribution with structural information, utilizing an iterative algorithm on heterogeneous graphs. The culmination of all these elements formed the ultimate similarity method. Experiments indicated that HetFS effectively handled ad-hoc queries with swift response times and comprehensive search outcomes, delivering commendable results compared to state-of-the-art methods in tasks such as link prediction, node clustering, and node classification. Future work involves adapting this heterogeneous graph mining framework to mine complex relation similarity from HINs.

\bmhead{Acknowledgments}

This work was mainly supported by the National Natural Science Foundation of China (NSFC No. 61732004).

\bmhead{Funding}

This work was mainly supported by the National Natural Science Foundation of China (NSFC No. 61732004).

% \backmatter

% \bmhead{Supplementary information}

% If your article has accompanying supplementary file/s please state so here. 

% Authors reporting data from electrophoretic gels and blots should supply the full unprocessed scans for key as part of their Supplementary information. This may be requested by the editorial team/s if it is missing.

% Please refer to Journal-level guidance for any specific requirements.

% \bmhead{Acknowledgments}

% Acknowledgments are not compulsory. Where included they should be brief. Grant or contribution numbers may be acknowledged.

% Please refer to Journal-level guidance for any specific requirements.

\section*{Statements and Declarations}

\begin{itemize}
\item This work was mainly supported by the National Natural Science Foundation of China (NSFC No. 61732004).
\item The authors declared no potential conﬂicts of interest with respect to the research, authorship, and/or publication of this article.
\item Ethics approval not applicable.
\item Some data used in this study were obtained from the publicly accessible websites described in Section 5. Other data are available from the corresponding author upon reasonable request.
\item Code is available from the corresponding author upon reasonable request.
\item Xuqi Mao: conceptualization of this study, methodology, software. Zhenyi Chen: conceptualization, methodology, software. Zhenying He: supervision. X. Sean Wang:
conceptualization, methodology.
\end{itemize}

\bibliography{ref}% common bib file
% \bibliography{sn-bibliography}

\end{document}